\documentclass[conference]{IEEEtran}
\IEEEoverridecommandlockouts
\usepackage{cite}
\usepackage{amsmath,amssymb,amsfonts}
\usepackage{algorithmic}
\usepackage{graphicx}
\usepackage{textcomp}
\usepackage{xcolor}

\usepackage{tikz}
\usetikzlibrary{shapes}
\usetikzlibrary{arrows}
\usetikzlibrary{automata}
\usetikzlibrary{fit}
\usetikzlibrary{shadows}
\usetikzlibrary{shapes}
\usetikzlibrary{trees}
\usetikzlibrary{positioning}
\usetikzlibrary{patterns}
\usetikzlibrary{backgrounds}
\usetikzlibrary{matrix}
\usetikzlibrary{external}

\usepackage{pgfplots}
\pgfplotsset{compat=1.18}

\tikzstyle{every state}=[draw=black,text=black,inner color= white,outer color= white,draw= black,text=black]
\tikzstyle{place}=[thick,draw=sthlmBlue,fill=blue!20,minimum size=8mm, opacity=.5]
\tikzstyle{red place}=[square,place, draw=sthlmRed, fill=sthlmLightRed]
\tikzstyle{green place}=[diamond, place, draw=sthlmGreen, fill=sthlmLightGreen]
\tikzset{chance/.style={state,place}}
\tikzset{maxnod/.style={state,red place}}
\tikzset{minnod/.style={state,green place}}
\tikzset{termin/.style={align=left}}

	\definecolor{sthlmLightBlue}{RGB}{214,237,252} % HEX #d6edfc
		
	\definecolor{sthlmBlue}{RGB}{0,110,191} % HEX #006ebf
		
	\definecolor{sthlmLightGreen}{RGB}{213,247,244} % HEX #0d5f7f4
		
	\definecolor{sthlmGreen}{RGB}{0,134,127} % #00867f

	\definecolor{sthlmLightGrey}{RGB}{213,217,225} %HEX #D5D9E1
		
	\definecolor{sthlmGrey}{RGB}{245,243,238} % HEX #f5f3ee
		
	\definecolor{sthlmDarkGrey}{RGB}{51,51,51} % HEX #333333

	\definecolor{sthlmLightOrange}{RGB}{255,215,210} % HEX #ffd7d2
		
	\definecolor{sthlmOrange}{RGB}{221,74,44} % HEX #dd4a2c

	\definecolor{sthlmLightPurple}{RGB}{241,230,252} % HEX #f1e6fc
		
	\definecolor{sthlmPurple}{RGB}{93,35,125} % HEX #5d237d

	\definecolor{sthlmLightRed}{RGB}{254,222,237} % HEX #c40064
		
	\definecolor{sthlmRed}{RGB}{196,0,100} % HEX #fedeed
		
	\definecolor{sthlmYellow}{RGB}{252,191,10} % HEX #fcbf0a

\usepackage[framemethod=TikZ]{mdframed}

\usepackage{subcaption}
\usepackage{alphalph}
\usepackage{amsmath}
\usepackage{amsthm}

\usepackage{xspace}
\def\ie{{\em i.e.}\xspace}

\usepackage{thmtools}
\usepackage{thm-restate}
\usepackage{cleveref}

\declaretheorem[name=\textbf{Proposition}]{proposition} 
\declaretheorem[name=Definition]{definition}
\usepackage{array,multirow,makecell}
\usepackage{paralist}
\usepackage{cancel}
\usepackage{diagbox}
\usepackage{graphicx}
\usepackage{booktabs}       % professional-quality tables
\usepackage{rotating}
\usepackage{algorithm2e}
\usepackage{bbm}
\newcommand{\PerfectAlgo}{PerfectAlgo}

\newcommand{\score}{score}
\newcommand{\applyMove}{ApplyMove}

\newcommand{\noop}{\constante{noop}}
\newcommand{\paper}{\constante{Paper}}
\newcommand{\scissor}{\constante{Scissors}}
\newcommand{\rock}{\constante{Rock}}

\newcommand{\sampling}{\textalgo{Sampling}}
\newcommand{\obsplays}{\constante{Play}}

\newcommand{\timestep}{t}
\newcommand{\agent}{i}

\newcommand{\Time}[2]{#1^{#2}}
\newcommand{\AgentPov}[2]{#1_{#2}}
\newcommand{\TimeAgentPov}[3]{#1^{#2}_{#3}}

\newcommand{\games}{G}
\newcommand{\nbplayer}{\mathcal{N}}

\newcommand{\worldstateTime}[1]{\Time{w}{#1}}
\newcommand{\worldstate}{\worldstateTime{}}
\newcommand{\worldstateSet}{\mathcal{W}}

\newcommand{\actionTimePov}[2]{\TimeAgentPov{a}{#1}{#2}}
\newcommand{\actionTime}[1]{\actionTimePov{#1}{}}
\newcommand{\action}{\actionTimePov{}{}}
\newcommand{\actionSetPov}[1]{\AgentPov{\mathcal{A}}{#1}}
\newcommand{\actionSet}{\actionSetPov{}}
\newcommand{\actionSetPovCond}[2]{\actionSetPov{#1} (#2)}

\newcommand{\infostateTimePov}[2]{\TimeAgentPov{s}{#1}{#2}}
\newcommand{\infostate}{\infostateTimePov{}{}}
\newcommand{\infostateSetPov}[1]{\AgentPov{\mathcal{I}}{#1}}
\newcommand{\infostateSet}{\infostateSetPov{}}
\newcommand{\infostateSetPovCond}[2]{\infostateSetPov{#1} (#2)}

\newcommand{\obsPrivTimePov}[2]{\TimeAgentPov{o}{#1}{#2}}

\newcommand{\obsPrivSetPov}[1]{\AgentPov{\mathcal{O}}{#1}}

\newcommand{\obsPrivSetPovCond}[2]{\obsPrivSetPov{#1} (#2)}

\newcommand{\budget}{budget}

\newcommand{\depth}{d}
\newcommand{\subgame}{U}

\newcommand{\policyTimePov}[2]{\TimeAgentPov{\pi}{#1}{#2}}
\newcommand{\policy}{\policyTimePov{}{}}
\newcommand{\policyTimePovCond}[4]{\policyTimePov{#1}{#2} (#3|#4)}
\newcommand{\policySetTimePov}[2]{\TimeAgentPov{\Pi}{#1}{#2}}
\newcommand{\policySet}{\policySetTimePov{}{}}
\newcommand{\policySetTimePovCond}[3]{\policySetTimePov{#1}{#2} (#3)}

\newcommand{\textalgo}[1]{\mathrm{{#1}}}
\newcommand{\constante}[1]{\mathrm{{#1}}}

\newcommand{\historyTime}[1]{\Time{h}{#1}}
\newcommand{\historyTimeSet}[1]{\Time{\mathcal{H}}{#1}}
\newcommand{\historyTimeSetCond}[2]{\historyTimeSet{#1} (#2)}

\def\BibTeX{{\rm B\kern-.05em{\sc i\kern-.025em b}\kern-.08em
    T\kern-.1667em\lower.7ex\hbox{E}\kern-.125emX}}
\begin{document}

\title{Perfect Information Monte Carlo with Postponing Reasoning
}

\author{\IEEEauthorblockN{Jérôme Arjonilla}
\IEEEauthorblockA{\textit{Université Paris Dauphine - PSL} \\
Paris, France \\
jerome.arjonilla@hotmail.fr}
\and
\IEEEauthorblockN{Abdallah Saffidine}
\IEEEauthorblockA{\textit{Potassco Solutions} \\
Potsdam, Germany \\
abdallah.saffidine@gmail.com}
\and
\IEEEauthorblockN{Tristan Cazenave}
\IEEEauthorblockA{\textit{Université Paris Dauphine - PSL} \\
Paris, France \\
tristan.cazenave@dauphine.psl.eu}
}

\IEEEoverridecommandlockouts
%\IEEEpubid{\makebox[\columnwidth]{ 979-8-3503-5067-8/24/\$31.00~\copyright2024 IEEE \hfill} %–> insert the copyright option applicable from above.
%\hspace{\columnsep}\makebox[\columnwidth]{ }}
% Add the following code after the \\make title command:
\IEEEpubidadjcol

\maketitle

\begin{abstract}

    Imperfect information games, such as Bridge and Skat, present challenges due to state-space explosion and hidden information, posing formidable obstacles for search algorithms. 
    Determinization-based algorithms offer a resolution by sampling hidden information and solving the game in a perfect information setting, facilitating rapid and effective action estimation. 
    However, transitioning to perfect information introduces challenges, notably one called strategy fusion.
    This research introduces `Extended Perfect Information Monte Carlo' (EPIMC), an online algorithm inspired by the state-of-the-art determinization-based approach Perfect Information Monte Carlo (PIMC). 
    EPIMC enhances the capabilities of PIMC by postponing the perfect information resolution, reducing alleviating issues related to strategy fusion. 
    However, the decision to postpone the leaf evaluator introduces novel considerations, such as the interplay between prior levels of reasoning and the newly deferred resolution.  
    In our empirical analysis, we investigate the performance of EPIMC across a range of games, with a particular focus on those characterized by varying degrees of strategy fusion. 
    Our results demonstrate notable performance enhancements, particularly in games where strategy fusion significantly impacts gameplay.
    Furthermore, our research contributes to the theoretical foundation of determinization-based algorithms addressing challenges associated with strategy fusion.%, thereby enhancing our understanding of these algorithms within the context of imperfect information game scenarios.

\end{abstract}

\begin{IEEEkeywords}
Imperfect Information Games, Search Algorithm, Determinization, Strategy Fusion
\end{IEEEkeywords}

\section{Introduction}

Search algorithms in artificial intelligence have significantly evolved, demonstrating superhuman performance in games such as Chess, Go~\cite{Silver2016MasteringTG}, Poker~\cite{brown_superhuman_2019}, Skat~\cite{furtak_recursive_2013}, and Contract Bridge\cite{bouzy_recursive_2020}. 
Perfect information games, like Chess and Go, where all information is available, have been extensively studied, allowing algorithms to surpass human professionals~\cite{Silver2016MasteringTG, silver_mastering_2017, silver_general_2018}. 
In contrast, imperfect information games, including Poker, Skat, and Bridge, where some information is hidden, have been less studied, with only a few algorithms capable of outperforming professional human players~\cite{brown_superhuman_2019}.

In imperfect information games, two commonly used search methods are regret-based approaches, which excel in Poker and are theoretically convergent but slower~\cite{neller2013introduction, tammelin_solving_2015, lanctot_monte_2009}, and determinization-based methods, considered state-of-the-art in various trick-taking card games, offering scalability but lacking theoretical guarantees~\cite{long_understanding_2010, furtak_recursive_2013, cowling_information_2012, cazenave_alpha_2021}. 
In recent years, both methods have incorporated neural networks to enhance performance and facilitate scalability on large games~\cite{moravcik_deepstack_2017, brown_superhuman_2019, jiang_deltadou_2019, zhang_combining_2021,schmid_player_2021}.

Determinization-based algorithms, like PIMC, operate by sampling hidden information based on current knowledge and using a \emph{perfect information leaf evaluator} to predict game outcomes under perfect information assumptions. 
These algorithms achieve state-of-the-art performance because solving problems with perfect information is inherently simpler than dealing with imperfect information. 
% However, they face a significant challenge known as 'strategy fusion'~\cite{frank_search_1998, long_understanding_2010}, which arises when the perfect information leaf evaluator independently solves each possible world without considering the uncertainties of games with imperfect information.
Despite their state-of-the-art performance, determinization-based algorithms face challenges, notably encountering `strategy fusion'~\cite{frank_search_1998, long_understanding_2010}.
This challenge arises from the use of the perfect information leaf evaluator, which independently solves each possible world without considering the uncertainties induced by games with imperfect information.

Within determinization-based algorithms, Perfect Information Monte Carlo (PIMC)~\cite{long_understanding_2010} is particularly susceptible to strategy fusion due to its early usage on the perfect information leaf evaluator in decision-making. 
Our study addresses this challenge by postponing the leaf evaluator until a depth $\depth$, which mitigates the impact of strategy fusion. 
The act of postponing the perfect information leaf evaluator at a depth of $\depth$ introduces new considerations, specifically, it prompts the need for alternative strategies to reason from step $1$ to step $\depth$.
We formally define the problem of strategy fusion, demonstrating that, in the worst case, increasing depth $\depth$ does not exacerbate strategy fusion, and in every case, there exists a depth $\depth$ that strictly reduces it. 
For finite games, there is a depth $\depth$ that eliminates strategy fusion entirely.

Section~\ref{sec:notAndBack} covers notation and a detailed explanation of determinization-based algorithms, specifically addressing PIMC and the strategy fusion challenge. 
Section~\ref{sec:E} introduces Extended PIMC, incorporating our idea of postponing the leaf evaluator at depth $\depth$, which operates online without initial costs, making it suitable for diverse games or General Game Playing~\cite{schofield_general_2019}. 
It is noteworthy that, like other modern determinization algorithms, there exists the potential to integrate neural networks for enhanced performance.
Section~\ref{sec:theoretical} presents the theoretical results of increasing depth $\depth$ in determinization-based algorithms. 
Section~\ref{sec:Results} showcases experimental findings across various games, highlighting cases with significant strategy fusion where deeper reasoning improves performance beyond other state-of-the-art methods. 
Section~\ref{sec:RelatedWork} reviews related research, and Section~\ref{sec:conclusion} summarizes our contributions and future research directions.

\newcommand{\wa}{\worldstateTime{a}}
\newcommand{\wb}{\worldstateTime{b}}
\newcommand{\wc}{\worldstateTime{c}}
\newcommand{\wwd}{\worldstateTime{d}}

\begin{figure}[!htbp]
    \centering

    \begin{tikzpicture}[->,-latex,auto,node distance=2cm,semithick, square/.style={regular polygon,regular polygon sides=4}] %
    
    \node[maxnod] (n1) at (0,0.6) {$\wa$};
    
    \node[minnod] (n2) at (-2.5,-2) {$\wb$};
    \node[minnod] (n3) at (0,-2) {$\wc$};
    \node[minnod] (n4) at (2.5,-2) {$\wwd$};
    
    \node[termin] (b1) at (-4,-2) {$-0.6$};
    \node[termin] (b2) at (-3.25,-3.5) {$0$};
    \node[termin] (b3) at (-2.5,-4) {$-1$};
    \node[termin] (b4) at (-1.75,-3.5) {$1$};
    \node[termin] (b5) at (-0.75,-3.5) {$1$};
    \node[termin] (b6) at (0,-4) {$0$};
    \node[termin] (b7) at (0.75,-3.5) {$-1$};
    \node[termin] (b8) at (1.75,-3.5) {$-1$};
    \node[termin] (b9) at (2.5,-4) {$1$};
    \node[termin] (b10) at (3.25,-3.5) {$0$};
    
    \path[very thin] 
        (n2) edge  [out=40, in=140, dashed,-, green] node {} (n3)
        (n3) edge  [out=40, in=140, dashed,-, green] node {} (n4);

    \draw (n1) -- node[midway,left] {R} ++ (n2);
    \draw (n1) -- node[midway,left] {P} ++(n3);
    \draw (n1) -- node[midway,left] {S} ++(n4);
    \draw (n1) -- node[midway,above] {L} ++(b1);

    \draw (n2) -- node[midway,left] {R} ++ (b2);
    \draw (n2) -- node[midway,left] {P} ++(b3);
    \draw (n2) -- node[midway,right] {S} ++ (b4);
    
    \draw (n3) -- node[midway,left] {R} ++ (b5);
    \draw (n3) -- node[midway,left] {P} ++(b6);
    \draw (n3) -- node[midway,right] {S} ++ (b7);
    
    \draw (n4) -- node[midway,left] {R} ++ (b8);
    \draw (n4) -- node[midway,left] {P} ++ (b9);
    \draw (n4) -- node[midway,right] {S} ++ (b10);

    \end{tikzpicture}

    \caption{Variant of `Rock-Paper-Scissors'. The red/green square/diamond is the first/second player and the dashed line represents worlds indistinguishable by the second player.}
~\label{fig:example}
\end{figure}
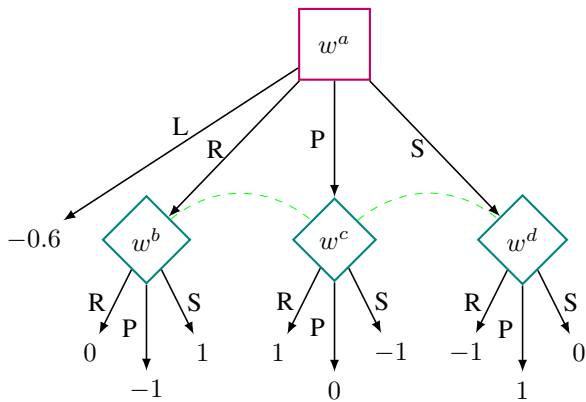

\section{Notation and Background}
~\label{sec:notAndBack}
\subsection{Notation}
~\label{sub:notation}

Throughout the paper, we use the formalism of Factored-Observation Stochastic Games (FOSG)\cite{kovarik_rethinking_2022} and utilize Figure~\ref{fig:example}, a variant of `Rock-Paper-Scissors' to present the notations. 
Our notation employs subscripts to denote players.

A game, denoted as $\games$, involves $\nbplayer$ players, initiating at $\worldstateTime{init}$ and progressing through successor world states represented by $\worldstate \in \worldstateSet$. 
Players make joint actions, denoted as $\actionTime{}=(\actionTimePov{}{1},\dots, \actionTimePov{}{\nbplayer}) \in \actionSetPovCond{}{\worldstate}$, in each world state $\worldstate$, and the game continues until a terminal state is reached. 
Upon choosing a joint action $\actionTime{}$, players observe rewards, and the next world state $\worldstateTime{\prime}$ is determined probabilistically. 
During this transition from $\worldstateTime{}$ to $\worldstateTime{\prime}$, each player receives an observation denoted as $\obsPrivTimePov{}{\agent} \in \obsPrivSetPovCond{\agent}{\worldstateTime{},\actionTime{},\worldstateTime{\prime}}$, where $\obsPrivSetPovCond{\agent}{\worldstateTime{},\actionTime{},\worldstateTime{\prime}}$ represents the set of possible observations for player $\agent$.
\\\\
In Figure~\ref{fig:example}, the first player faces the choice between `Leave' with a reward of $-0.6$ or engaging in the game. 
If it opts to play, the standard rules apply: $\rock$ beats $\scissor$, $\scissor$ beats $\paper$, and $\paper$ beats $\rock$, with wins yielding $1$, losses resulting in $-1$, and ties providing $0$.
The game encompasses four possible world states denoted as $\worldstateSet={\wa,\wb,\wc,\wwd}$, where $\wa$ is the initial state and $\wb$, $\wc$, and $\wwd$ follow the first action. 

The first player has four possible actions in $\wa$ and a null action in other states—specifically, $\actionSetPovCond{1}{\wa}=$\{$\textalgo{Leave}$, $\rock$, $\paper$, $\scissor$\} and $\actionSetPovCond{1}{\wb}=\actionSetPovCond{1}{\wc}=\actionSetPovCond{1}{\wwd}={\noop}$. 
The second player has a null action in $\wa$ and three actions in other states—namely, $\actionSetPovCond{2}{\wa}={\emptyset}$ and $\actionSetPovCond{2}{\wb}=\actionSetPovCond{2}{\wc}=\actionSetPovCond{2}{\wwd}={\rock,\paper,\scissor}$.
In this game, the second player receives the observation $\obsplays$ if the first player plays $\paper$, $\scissor$, or $\rock$. Rewards for both players are obtained at the game's conclusion.
\\\\
A \emph{history} is a finite sequence of world states and legal actions denoted as $ \historyTime{\timestep} = (\worldstateTime{0},\actionTimePov{0}{},\worldstateTime{1},\actionTimePov{1}{}, \ldots, \worldstateTime{t})$. 
An \emph{infostate} $\infostateTimePov{}{\agent}$ is a sequence of an agent’s observations and actions, denoted as $\infostateTimePov{\timestep}{\agent} = (\actionTimePov{0}{\agent}, \obsPrivTimePov{1}{\agent}, \actionTimePov{1}{\agent}, \dots, \obsPrivTimePov{\timestep-1}{\agent})$.
For each history, there exists a unique infostate noted as $\infostateTimePov{}{\agent}(\historyTime{})$, and for each infostate, there is a set of histories that match the sequence denoted as $\historyTimeSetCond{}{\infostate}$. 
$\infostateSetPovCond{}{\games}$ represents the set of possible infostates in the game $\games$, and $\policySetTimePovCond{}{}{\games}$ is the set of possible policies in the game $\games$, where a \emph{policy} $\policy$ is a function mapping every history $\historyTime{}$ to a probability distribution over actions.
\\\\
In Figure~\ref{fig:example}, if we examine the history $\historyTime{1}=(\wa,\rock,\wb)$, considering it from the second player's perspective, the infostate becomes $\infostateTimePov{1}{2}=(\noop,\obsplays)$ since no action was taken by the second player, and the observation $\obsplays$ was noted. 
The set of histories that correspond to $\infostateTimePov{1}{2}$ includes \{$(\wa,\rock,\wb)$; $(\wa,\paper,\wc)$; $(\wa,\scissor,\wwd)$\}.

\subsection{Determinization-based algorithm}
~\label{sub:pimc}

Each determinization-based algorithm has its own characteristics, nevertheless, they share some common features such as (\romannumeral1) \emph{samples} a history $\historyTime{}$ according to a probability distribution over the current infostate $\infostate$; (\romannumeral2) uses a \emph{perfect information leaf evaluator} for estimating the value of the sampled world state. 
In the description of the Algorithm~\ref{alg:pimc} and~\ref{alg:PIMC2} (\romannumeral1) is noted $\sampling(\infostate)$ and (\romannumeral2) is noted $\PerfectAlgo(\historyTime{})$.

A perfect information leaf evaluator is an algorithm used in games with perfect information to estimate the value of a history $\historyTime{}$. 
It predicts the outcome of a game from a specific position. %, assuming optimal moves by both players. 
This evaluator can be exhaustive methods like Minimax with Alpha-Beta pruning~\cite{knuth_analysis_1975}, heuristics methods like Random Rollout (also called \emph{playouts}~\cite{browne_survey_2012}) or neural network~\cite{Silver2016MasteringTG, silver_mastering_2017, silver_general_2018}.

Determinization-based algorithms are simple and, in practice, achieve great results. Yet, certain problems are encountered such as (\romannumeral1) non-locality and strategy-fusion~\cite{frank_search_1998,long_understanding_2010}; (\romannumeral2) revealing private hidden information~\cite{frank_search_1998,arjonilla2023mixture}; (\romannumeral3) no theoretical guarantees. 
\\\\
In Figure~\ref{fig:example}, when conducting sampling from the infostate $\infostateTimePov{1}{2}$, the available options include sampling $\wb$, $\wc$, or $\wwd$. 
Utilizing a perfect information leaf evaluator, such as Minimax, on the world state $\wb$ would yield a value of $-1$, as the second player optimally plays $\paper$ to maximize their score \ie, minimizes the value of the first player.

\subsection{Strategy fusion}
~\label{sub:strat_fusion}

In imperfect information games, histories stemming from the same infostate must be approached with the same strategy, as players cannot distinguish between them. 
Formally, $\forall \infostate \in \infostateSetPovCond{}{\games}, \forall \historyTime{},\historyTime{'} \in \historyTimeSetCond{}{\infostate}, \policyTimePovCond{}{}{\cdot}{\historyTime{}} = \policyTimePovCond{}{}{\cdot}{\historyTime{'}}$.

However, determinization-based algorithms deviate from this principle. 
They employ a perfect information leaf evaluator algorithm to estimate the value at each sampled history. 
In other words, each history originating from the same infostate is solved using a strategy tailored to that specific history. 
Formally, $\forall \infostate \in \infostateSetPovCond{}{\games}, \forall \historyTime{},\historyTime{'} \in \historyTimeSetCond{}{\infostate}, \policyTimePovCond{}{}{\cdot}{\historyTime{}} \neq \policyTimePovCond{}{}{\cdot}{\historyTime{'}}$.
\\\\
In Figure~\ref{fig:example}, the strategies in $\wb$, $\wc$, and $\wwd$ must be identical since they originate from the same infostate $\infostateTimePov{1}{2}$. For the second player, the optimal strategy results in an average score of $0$ (playing with the same probability for all three actions). 
By back-propagating, the first player opts for $\obsplays$, leading to an average score of $0$, instead of choosing $\textalgo{Leave}$, which results in $-0.6$.

From the perspective of a determinization-based algorithm, a policy is tailored to the sampled history. 
The best policy for the second player is $\paper$ in $\wb$, $\scissor$ in $\wc$, and $\rock$ in $\wwd$. 
In $\wb$, $\wc$, and $\wwd$, playing the best policy yields $-1$. 
By back-propagating, the first player concludes that opting for $\mathrm{Leave}$ to obtain $-0.6$ is preferable compared to playing and receiving $-1$.

\subsection{Perfect Information Monte Carlo}

Perfect Information Monte Carlo (PIMC) is a determinization-based algorithm that is the state-of-the-art of many imperfect information games.

\RestyleAlgo{ruled}
\begin{algorithm}
\caption{PIMC}\label{alg:pimc}
\SetKwFunction{FMain}{PIMC}
\SetKwProg{Fn}{Function}{:}{}
\Fn{\FMain{$\infostate $}}{
    \For{$\action$ $\in$ $\actionSet$ ($\infostate)$}{
        $\score$[$\action$] $\gets$ $0$\;
    }
    \While{$\budget$} {
        $\worldstate$ $\gets$ $\sampling(\infostate)$\;
        
        \For{$\action$ $\in$ $\actionSet$ ($\worldstate$)}{
            $\worldstateTime{\prime}$ $\gets$ $\worldstate.\applyMove(\action)$\;
            
            $\score$ [$\action$] $\gets$ $\score$[$\action$] + $\PerfectAlgo(\worldstateTime{\prime}$)\;
        }
    }
    \Return Returns the best action on average\;
}
\end{algorithm}

PIMC is defined in Algorithm~\ref{alg:pimc} and works as follows (\romannumeral1) samples a world state $\worldstate$ according to the probability distribution over the current infostate $\infostate$; (\romannumeral2) plays an action $\action$ on the world state $\worldstate$, and observes the next world state $\worldstateTime{\prime}$; (\romannumeral3) estimates the world state $\worldstateTime{\prime}$ by using the perfect information leaf evaluator; (\romannumeral4) repeats until the $\budget$ is over; (\romannumeral5) selects the action that produces the best results in average. 

\section{Extended PIMC}
~\label{sec:E}

As presented, employing the perfect information leaf evaluator results in strategy fusion. 
In the case of PIMC, this evaluator is utilized after the first action is played.
However, there are no inherent constraints preventing the resolution after the first action, and it is not difficult to believe that postponing the use of the perfect information leaf evaluator could mitigate the issue of strategy fusion.

In the following, we introduce a novel algorithm that embraces this straightforward concept of postponing the leaf evaluator's utilization. 
The algorithm, termed `EPIMC' (Extended Perfect Information Monte Carlo), extends the PIMC paradigm by incorporating Extended reasoning at a depth of $\depth$, where the instance of $\depth=1$ corresponds to PIMC.

\newcommand{\ExplorationStrategy}{\textalgo{RandomAction}}

\RestyleAlgo{ruled}
\begin{algorithm}
\caption{Extended PIMC}\label{alg:PIMC2}
\SetKwFunction{FMain}{$\textalgo{Extended PIMC}$}
\SetKwProg{Fn}{Function}{:}{}
\Fn{\FMain{$\depth$, $ \infostate$}}{

    %Create $u$\;
    Create the game $\subgame$ and $u$, the initial world state\; 
    
    \While{$\budget$}{
        $\worldstate$ $\gets$ $\sampling(\infostate)$\;      $\textalgo{Query}$({$\subgame$,$u$, $\worldstate$, $\depth$})\;
    }
    
    \Return $\textalgo{ImperfectAlgo}$($\subgame$)\;
}

\vspace{5pt}

\SetKwFunction{FMain}{$\textalgo{Query}$}
\SetKwProg{Fn}{Function}{:}{}
\Fn{\FMain{$\subgame$, $u$, $\worldstate$, $\depth$}}{

    \If{$\depth == 0$ or $ \worldstate.\textalgo{IsTerminal}()$}
    {
        $u.\textalgo{value}$ $\gets$ $u.\textalgo{value}$ +  $\textalgo{\PerfectAlgo}$($\worldstate$)\;
        \Return \;
    }

    \iffalse
    \For{$\textalgo{\action}$ $\in$ $\textalgo{\ExplorationStrategy}$($\textalgo{\worldstate})$}
    {
        $\textalgo{\worldstate}^\prime$ $\gets$ $\textalgo{\worldstate}.\applyMove(\action)$\;
        Get $\textalgo{u}^\prime$ associated with $\textalgo{\worldstate}^\prime$\;
        Update dynamics in $\textalgo{\subgame}$ according to $\textalgo{u}^\prime$ and $\textalgo{u}$\;
        $\textalgo{Query}$({$\textalgo{\subgame}$, $\textalgo{u^\prime}$, $\textalgo{\worldstateTime{\prime}}$, $\textalgo{\depth -1}$})\;
    }
    \fi

    $\worldstateTime{\prime}$ $\gets$ $\worldstate.\applyMove(\ExplorationStrategy(\worldstate))$\;
    Get $u^\prime$ associated with $\worldstateTime{\prime}$\;
    Update dynamics in $\subgame$ according to $u^\prime$ and $u$\;
    $\textalgo{Query}$({$\subgame$, $u^\prime$, $\worldstateTime{\prime}$, $\depth -1$})\;
}

\end{algorithm}

The pseudo-code is available in Algorithm~\ref{alg:PIMC2} and works as follows (\romannumeral1) creates a subgame game $\subgame$; (\romannumeral2) samples a world state $\worldstate$ according to the probability distribution over the current infostate $\infostate$; (\romannumeral3) plays an action $\action$ on $\worldstate$, observes the next world state $\worldstateTime{\prime}$ continue until $\worldstateTime{\prime}$ is a terminal node or the depth $\depth$ is obtained, and updates the dynamic in subgame $\subgame$; (\romannumeral4) estimates the world state $\worldstateTime{\prime}$ by using the perfect information leaf evaluator; (\romannumeral5) repeats phases 2 to 4 until the $\budget$ is over; (\romannumeral6) solves the subgame $\subgame$ which have been created during phase 2 to 5 by using an algorithm that does not create strategy fusion (\ie an algorithm that do not use a perfect leaf evaluator to resolve the subgame $\subgame$).

In the following, a few points are clarified.

\subsection*{Exploration Strategy}

% Importantly, Algorithm~\ref{alg:PIMC2} incorporates $\ExplorationStrategy$, a feature absent in the original PIMC Algorithm~\ref{alg:pimc}. 
In the original PIMC, world state evaluations occur a maximum of $|A|$ times during each sampling phase, with $|A|$ denoting the highest feasible action count. 
However, applying the same approach at depth $d$ could potentially lead to evaluating $|A|^d$ world states. 
Therefore, careful consideration of the number of actions per step becomes imperative.
In EPIMC, our chosen strategy is to explore one action per sampling iteration, resulting in a singular world state evaluation per sampling instance, however, other choices could have been envisaged.

\subsection*{Depth}

While theoretically, any depth $\depth$ could be employed, practical considerations necessitate a prudent choice of $\depth$. 
Selecting a depth $\depth$ that is too small can exacerbate strategy fusion issues, as observed in approaches like PIMC. 
On the other hand, opting for a depth $\depth$ that is too large demands significant sampling efforts to accurately estimate the subgame $\subgame$. 
Moreover, depending on the algorithm employed to solve the subgame $\subgame$, a larger depth can lead to increased computational costs.

\subsection*{Subgame Resolution}

By postponing the application of the perfect information leaf evaluator until a depth of $\depth$, alternative reasoning methods must be employed for steps $1$ through $\depth$. 
In EPIMC, the subgame $\subgame$ of size $\depth$ is built to approximate the real game by encapsulating various elements, including infostates, world states, post-action dynamics, and more. 
In particular, the leaves of the subgame $\subgame$ are average scores obtained from the leaf evaluator. 

After the budget is finished, the subgame $\subgame$ is solved using an algorithm that does not create strategy fusion. 
In other words, one needs an algorithm that works on infostates instead of world states. 
This choice of algorithm, although potentially resource-intensive, carries a reduced computational burden when applied to a subgame $\subgame$ of size $\depth$.
One can think of using information set search~\cite{parker2006overconfidence, parker2006paranoia} which operates on infostates according to a minimax rule or CFR/CFR+~\cite{neller2013introduction, tammelin_solving_2015} that have theoretical guarantees for two players.

\section{Theoretical foundation}
~\label{sec:theoretical}

In the following, we present the theoretical foundation for determinization-based algorithms that suffer from strategy fusion.
We formally (\romannumeral1) define the condition to \emph{create} strategy fusion; (\romannumeral2) define the quantity of strategy fusion; (\romannumeral3) prove that, in the worst case, increasing the depth does not increase the strategy fusion, and in every case, there exists a depth $\depth$ such that the strategy fusion is strictly reduced and in a finite game, there exists a depth $\depth$ such as the that the strategy fusion is removed.

\begin{definition}
~\label{def:create_sf}
For any game $\games$, a policy $\policy{}{} \in \policySet(\games)$ and an infostate $s \in \infostateSet(\games)$ \textbf{create}s strategy fusion if there are $\historyTime{},\historyTime{\prime} \in \historyTimeSetCond{}{\infostate}$ such that $\policy{}{}(\historyTime{}) \neq \policy{}{}(\historyTime{\prime})$.
For any game $\games$, a policy $\policy{}{} \in \policySet(\games)$ \textbf{create}s strategy fusion if there is an infostate $\infostate \in \infostateSet(\games)$ such that $\policy{}{}$ creates strategy fusion in $\infostate$.
\end{definition}

To evaluate the quantity of strategy fusion in $\policy{}{} \in \policySet(\games)$, we propose the following measure $SF(\policy{}{}) = |\{s$ such that $\forall \infostate \in S(\games)$, $\policy{}{}$ \emph{create}s strategy fusion in $s$ $\}|$. 
In other, we count the number of infostate that create strategy fusion. $SF(\policy{}{})= 0$ implies  that there is no strategy fusion. 

In the subsequent discussions, for the sake of simplicity, we presume adherence to the policy provided by EPIMC, denoted as $\policySetTimePovCond{E^\depth}{}{\games}$ when executing EPIMC at a depth of $\depth$. 
Although alternative choices could have been considered, opting for the EPIMC policy is more straightforward, as influenced by the uniform growth of the EPIMC policy across the entire depth $\depth$ space.

\begin{proposition}
~\label{prop:inforequal}
$\forall \games, \forall \depth \in [0,T-1]$ where $T = \{T$ if $\games$ has a finite horizon $T$; else $\infty \}$,$\forall \policy{}{} \in \policySetTimePovCond{E^d}{}{\games}$, then  $\forall \policyTimePov{\prime}{} \in \policySetTimePovCond{E^{\depth+1}}{}{\games}$, $SF(\policy{}{},\games) \geq SF(\policyTimePov{\prime}{},\games)$ 
\end{proposition}

\begin{proof}

For any $\infostate \in \infostateSetPovCond{}{\games}$, $\infostate$ does not create strategy fusion in $\policyTimePov{1:\depth}{}$. 
Increasing the depth by $1$, extends the non-inducing region.
Two possibilities arise: (\romannumeral1) if at least one $\infostate$ at $\depth+1$ create strategy fusion, it can no longer do so, i.e., $SF(\policyTimePov{}{},\games) > SF(\policyTimePov{\prime}{},\games)$; (\romannumeral2) if all $\infostate$ at $\depth+1$ do not create strategy fusion, increasing the depth does not reduce SF, i.e., $SF(\policyTimePov{}{},\games) \geq SF(\policyTimePov{\prime}{},\games)$.
We conclude $SF(\policy{}{},\games) \geq SF(\policyTimePov{\prime}{},\games)$.
\end{proof}

\begin{proposition}
~\label{prop:inf}
$\forall \games, \forall \depth \in [0,T]$ where $T = \{T $ if $\games$ has a finite horizon $T$; else $\infty \}$,$\forall \policy{}{} \in \policySetTimePovCond{E^\depth}{}{\games}$, if $SF(\policy{}{},\games)>0$, then $\exists \depth^\prime \in [1,T-\depth], \forall \policyTimePov{\prime}{} \in \policySetTimePovCond{E^{\depth+\depth^\prime}}{}{\games}$ such that $SF(\policy{}{},\games) > SF(\policyTimePov{\prime}{},\games)$.
\end{proposition}

\begin{proof}
    Leveraging the rationale from Proposition~\ref{prop:inforequal}, increasing the depth by $\depth^\prime$ extends the non-inducing region. 
    Given the presence of strategy fusion ($SF(\policy,\games)>0$), at least one infostate creates strategy fusion. 
    Extending the reasoning up to this infostate, located $\depth^\prime$ away from the original depth $\depth$, effectively diminishes strategy fusion. 
    Hence, we establish $SF(\policy{}{},\games) > SF(\policyTimePov{\prime}{},\games)$.
\end{proof}

We extend Proposition~\ref{prop:inf} by showing that, in a finite game, there is a depth at which there is no longer strategy fusion.

\begin{proposition}
~\label{prop:finite}
    $\forall \games$ such that $\games$ has a finite horizon $T$, $\forall \depth \in [1,T-1]$,$\forall \policy{}{} \in \policySetTimePov{E^\depth}{}{\games}$ if $SF(\policy{}{},\games)>0$ then $\exists \depth^\prime \in [1,T-\depth], \forall \policyTimePov{\prime}{} \in \policySetTimePov{E^{\depth+\depth^\prime}}{}{\games}$ such that $SF(\policyTimePov{\prime}{},\games)=0$.
\end{proposition}

\begin{proof}
    Setting $\depth^\prime=T-\depth$, i.e., until the end of the game, ensures that no further infostate can induce strategy fusion, as there are no infostates remaining.
\end{proof}

\section{Results}
~\label{sec:Results}
\subsection{Games}

Our experiment set involved testing five games: Card Game, Battleship, Dark Chess, Phantom Tic-Tac-Toe, and Dark Hex. 
Each of them is considered a large game, is described below and is implemented in OpenSpiel~\cite{lanctot_openspiel_2019}, a collection of environments and algorithms for research in general reinforcement learning and search/planning in games.
The benchmarks were chosen to show the strengths and weaknesses of the algorithm, especially, by choosing games with public observations (Card game, Battleship) or with private observations (Phantom Tic-Tac-Toe, Dark Chess, and Dark Hex). 

Public observations refer to information that is visible to all players, while private observations are exclusive to a player. 
This distinction significantly influences the dynamics of strategy fusion. 
Private observations amplify the number of potential world states within a single infostate, thereby increasing the likelihood of strategy fusion. 
As our method reduces strategy fusion, we expect EPIMC to be more effective in games where observations are private.

\paragraph{Card game}
The game is played with two players, $22$ cards are taken from a pack of $52$, and known by all, $6$ are hidden and $8$ is given to each player. 
The playing phase is decomposed into tricks, the player starting the trick is the one who won the previous trick. 
The starting player of a trick can play any card in his hand, but the other player must follow the suit of the first player. 
If he can not, he can play any card he wants without possibly winning the trick. 
The winner of the trick is the one with the highest-ranking card. 
At the end of the game, a player wins if he has at least half of the number of tricks won. 

\paragraph{Battleship}

Battleship is a two-player strategy-type guessing game. 
Each player possesses a grid. 
In the beginning, each player secretly places a set of ships S on their grid. 
After placement, turn after turn, each player tries to fire at other players' ships. 
The game ends when all the ships of a player have been destroyed. 
The payoff of each player is computed as the sum of the opponent’s ships that were destroyed, minus the sum of ships that the player lost. 
The grid is fixed to $3\times3$, with $2$ ships, one of size $1\times 1$, and the second of size $2 \times 1$. 

\paragraph{Dark Chess}
Dark chess is a chess variant with incomplete information. 
In chess, there are two players, white and black, each controlling a set of chess pieces of their respective colors. 
The goal of the game is to checkmate the opponent's king. 
In Dark chess, the incomplete information comes from the fact that each player sees his own pieces, but only sees his opponent's pieces if they are reachable by one of his pieces. 
Furthermore, in his variant, the purpose is to capture the king (not to checkmate it), however, a player must be wary as he is not told if their king is in check. 
The size of the board is fixed to $4 \times 4$. 

\paragraph{Phantom Tic-Tac-Toe}
Phantom Tic-Tac-Toe is a variant of the game of Tic-Tac-Toe with imperfect information. 
In Tic-Tac-Toe, the goal is to claim three cells along the same row, column, or diagonal. With imperfect information, the players do not observe the other player's pieces, only a referee knows the world state of the board. 
When it is a player's turn, the player selects a move and indicates it to the referee. 
The referee informs the player's whether the action is `legal' or `illegal'. 
If the move is `illegal', the player must choose a new move until they find a legal one. 

\paragraph{Dark Hex}
Dark Hex is an imperfect information version of the classic game of Hex. 
The objective of the game is to create a connection between opposite sides of a rhombus-shaped board. 
In Dark Hex, players are not exposed to opposite sides' pieces of information. 
Only a referee has the full information of the board and when a move fails due to collision/rejection the player gets some information of the cell and is allowed to make another move until success. 
The size of the board is fixed to $4 \times 4$.

\subsection{Experimental Information}

For all the experiments, the budget ranged from $0.1$ seconds to $100$ seconds for Card Game, Battleship and Phantom Tic-Tac-Toe, and from $0.1$ seconds to $1000$ seconds for Dark Hex and Dark Chess. 
Each experiment was conducted over $500$ games, in which the games were evenly split between playing in the first and second positions. 
The opponent is PIMC with a fixed one-second budget.

All experiments were executed on a single CPU Intel(R) Xeon(R) Gold 5218. 
EPIMC and PIMC use random rollout as the perfect information leaf evaluator, depth at $3$ and Information Set Search for the subgame resolution.
In practice, PIMC is often tested with minimax as the perfect information leaf evaluator, yet using it may be slow in large benchmarks or require using a handmade heuristic.
As a reminder, EPIMC at depth $1$ is PIMC. 
To reduce EPIMC/PIMC costs, multiple CPUs could be utilized, but for fairness across algorithms, this approach was not employed.

\paragraph*{\textbf{Other online algorithms}}
The other algorithms compared are Information Set MCTS (IS-MCTS) \cite{cowling_information_2012}, Online Outcome Sampling (OOS) \cite{lisy_online_2015}, Recursive PIMC (IIMC) \cite{furtak_recursive_2013}, and a random agent (Random). 
IS-MCTS is a determinization-based algorithm that employs Monte Carlo Tree Search on infostates. 
OOS is a regret-based algorithm that converges to the Nash equilibrium with increasing search time. 
IIMC is a determinization-based algorithm rooted in PIMC, estimating action values by recursively calling PIMC until game completion.

For IS-MCTS, the exploration constant was chosen from ${0.6,1,1.5,2}$ and set at $1$. 
In OOS, the target was selected between Information Set and Public Subgame Targeting, and it was set at Information Set. 
Regarding IIMC, the number of samplings at level 2 was chosen from ${2,5,10}$ and set at $5$.

\subsection{Experimental Results}

In our experiments, we aimed (\romannumeral1) to analyze the hyperparameters of EPIMC (depth, subgame resolution, perfect leaf evaluator) and (\romannumeral2) to compare its performance against other online algorithms. 

\paragraph{Postponing leaf evaluator}

\begin{figure*}[!htbp]
\centering
    \subfloat[\centering Card Game]{\scalebox{0.6}{% This file was created with tikzplotlib v0.10.1.
\begin{tikzpicture}

\definecolor{darkgray176}{RGB}{176,176,176}
\definecolor{darkorange25512714}{RGB}{255,127,14}
\definecolor{forestgreen4416044}{RGB}{44,160,44}
\definecolor{lightgray204}{RGB}{204,204,204}
\definecolor{steelblue31119180}{RGB}{31,119,180}

\begin{axis}[
legend cell align={left},
legend style={fill opacity=0.8, draw opacity=1, text opacity=1, draw=lightgray204},
log basis x={10},
tick align=outside,
tick pos=left,
x grid style={darkgray176},
xlabel={Budget},
xmin=0.0707945784384138, xmax=141.253754462275,
xmode=log,
xtick style={color=black},
xtick={0.001,0.01,0.1,1,10,100,1000,10000},
xticklabels={
  \(\displaystyle {10^{-3}}\),
  \(\displaystyle {10^{-2}}\),
  \(\displaystyle {10^{-1}}\),
  \(\displaystyle {10^{0}}\),
  \(\displaystyle {10^{1}}\),
  \(\displaystyle {10^{2}}\),
  \(\displaystyle {10^{3}}\),
  \(\displaystyle {10^{4}}\)
},
y grid style={darkgray176},
ylabel={Winning rate},
ymin=0, ymax=100,
ytick style={color=black}
]
\path [draw=steelblue31119180, semithick]
(axis cs:0.1,46.5010670636492)
--(axis cs:0.1,52.6989329363508);

\path [draw=steelblue31119180, semithick]
(axis cs:0.3,47.0009740911054)
--(axis cs:0.3,53.1990259088946);

\path [draw=steelblue31119180, semithick]
(axis cs:1,46.700992685391)
--(axis cs:1,52.899007314609);

\path [draw=steelblue31119180, semithick]
(axis cs:3,46.4011228485143)
--(axis cs:3,52.5988771514857);

\path [draw=steelblue31119180, semithick]
(axis cs:10,47.0009740911054)
--(axis cs:10,53.1990259088946);

\path [draw=steelblue31119180, semithick]
(axis cs:30,46.30119103138)
--(axis cs:30,52.49880896862);

\path [draw=steelblue31119180, semithick]
(axis cs:100,47.0009740911054)
--(axis cs:100,53.1990259088946);

\path [draw=darkorange25512714, semithick]
(axis cs:0.1,45.4023627714014)
--(axis cs:0.1,51.5976372285986);

\path [draw=darkorange25512714, semithick]
(axis cs:0.3,47.50119103138)
--(axis cs:0.3,53.69880896862);

\path [draw=darkorange25512714, semithick]
(axis cs:1,46.5010670636492)
--(axis cs:1,52.6989329363508);

\path [draw=darkorange25512714, semithick]
(axis cs:3,44.8037024458234)
--(axis cs:3,50.9962975541766);

\path [draw=darkorange25512714, semithick]
(axis cs:10,43.2094646910284)
--(axis cs:10,49.3905353089716);

\path [draw=darkorange25512714, semithick]
(axis cs:30,43.0104095313456)
--(axis cs:30,49.1895904686544);

\path [draw=darkorange25512714, semithick]
(axis cs:100,43.2094646910284)
--(axis cs:100,49.3905353089716);

\path [draw=forestgreen4416044, semithick]
(axis cs:0.1,45.8017179495727)
--(axis cs:0.1,51.9982820504273);

\path [draw=forestgreen4416044, semithick]
(axis cs:0.3,46.1013645945352)
--(axis cs:0.3,52.2986354054648);

\path [draw=forestgreen4416044, semithick]
(axis cs:1,47.3010670636492)
--(axis cs:1,53.4989329363508);

\path [draw=forestgreen4416044, semithick]
(axis cs:3,48.3021829518191)
--(axis cs:3,54.4978170481809);

\path [draw=forestgreen4416044, semithick]
(axis cs:10,47.50119103138)
--(axis cs:10,53.69880896862);

\path [draw=forestgreen4416044, semithick]
(axis cs:30,47.7013645945352)
--(axis cs:30,53.8986354054648);

\path [draw=forestgreen4416044, semithick]
(axis cs:100,48.6027596515608)
--(axis cs:100,54.7972403484392);

\addplot [semithick, steelblue31119180, dotted, mark=*, mark size=3, mark options={solid}]
table {%
0.1 49.6
0.3 50.1
1 49.8
3 49.5
10 50.1
30 49.4
100 50.1
};
\addlegendentry{EPIMC\_D1}
\addplot [semithick, darkorange25512714, dotted, mark=pentagon*, mark size=3, mark options={solid}]
table {%
0.1 48.5
0.3 50.6
1 49.6
3 47.9
10 46.3
30 46.1
100 46.3
};
\addlegendentry{EPIMC\_D2}
\addplot [semithick, forestgreen4416044, dotted, mark=square*, mark size=3, mark options={solid}]
table {%
0.1 48.9
0.3 49.2
1 50.4
3 51.4
10 50.6
30 50.8
100 51.7
};
\addlegendentry{EPIMC\_D3}
\end{axis}

\end{tikzpicture}}}
    \subfloat[\centering Battleship]{\scalebox{0.6}{% This file was created with tikzplotlib v0.10.1.
\begin{tikzpicture}

\definecolor{darkgray176}{RGB}{176,176,176}
\definecolor{darkorange25512714}{RGB}{255,127,14}
\definecolor{forestgreen4416044}{RGB}{44,160,44}
\definecolor{lightgray204}{RGB}{204,204,204}
\definecolor{steelblue31119180}{RGB}{31,119,180}

\begin{axis}[
legend cell align={left},
legend style={fill opacity=0.8, draw opacity=1, text opacity=1, draw=lightgray204},
log basis x={10},
tick align=outside,
tick pos=left,
x grid style={darkgray176},
xlabel={Budget},
xmin=0.0707945784384138, xmax=141.253754462275,
xmode=log,
xtick style={color=black},
xtick={0.001,0.01,0.1,1,10,100,1000,10000},
xticklabels={
  \(\displaystyle {10^{-3}}\),
  \(\displaystyle {10^{-2}}\),
  \(\displaystyle {10^{-1}}\),
  \(\displaystyle {10^{0}}\),
  \(\displaystyle {10^{1}}\),
  \(\displaystyle {10^{2}}\),
  \(\displaystyle {10^{3}}\),
  \(\displaystyle {10^{4}}\)
},
y grid style={darkgray176},
ymin=0, ymax=100,
ytick style={color=black},
ytick=\empty
]
\path [draw=steelblue31119180, semithick]
(axis cs:0.1,46.601023676115)
--(axis cs:0.1,52.798976323885);

\path [draw=steelblue31119180, semithick]
(axis cs:0.3,43.8069299678151)
--(axis cs:0.3,49.9930700321849);

\path [draw=steelblue31119180, semithick]
(axis cs:1,46.2012716130645)
--(axis cs:1,52.3987283869355);

\path [draw=steelblue31119180, semithick]
(axis cs:3,46.4011228485143)
--(axis cs:3,52.5988771514857);

\path [draw=steelblue31119180, semithick]
(axis cs:10,44.006184856201)
--(axis cs:10,50.193815143799);

\path [draw=steelblue31119180, semithick]
(axis cs:30,49.7058309910414)
--(axis cs:30,55.8941690089586);

\path [draw=steelblue31119180, semithick]
(axis cs:100,45.6020155429699)
--(axis cs:100,51.7979844570301);

\path [draw=darkorange25512714, semithick]
(axis cs:0.1,44.5045400367635)
--(axis cs:0.1,50.6954599632365);

\path [draw=darkorange25512714, semithick]
(axis cs:0.3,46.900967893035)
--(axis cs:0.3,53.099032106965);

\path [draw=darkorange25512714, semithick]
(axis cs:1,45.2027596515608)
--(axis cs:1,51.3972403484392);

\path [draw=darkorange25512714, semithick]
(axis cs:3,43.7073212232759)
--(axis cs:3,49.8926787767242);

\path [draw=darkorange25512714, semithick]
(axis cs:10,47.201023676115)
--(axis cs:10,53.398976323885);

\path [draw=darkorange25512714, semithick]
(axis cs:30,44.7039692094554)
--(axis cs:30,50.8960307905446);

\path [draw=darkorange25512714, semithick]
(axis cs:100,43.9065511803167)
--(axis cs:100,50.0934488196833);

\path [draw=forestgreen4416044, semithick]
(axis cs:0.1,43.5081411571677)
--(axis cs:0.1,49.6918588428323);

\path [draw=forestgreen4416044, semithick]
(axis cs:0.3,42.7119206331443)
--(axis cs:0.3,48.8880793668557);

\path [draw=forestgreen4416044, semithick]
(axis cs:1,44.903448111205)
--(axis cs:1,51.096551888795);

\path [draw=forestgreen4416044, semithick]
(axis cs:3,44.006184856201)
--(axis cs:3,50.193815143799);

\path [draw=forestgreen4416044, semithick]
(axis cs:10,45.0032062025379)
--(axis cs:10,51.1967937974621);

\path [draw=forestgreen4416044, semithick]
(axis cs:30,44.7039692094554)
--(axis cs:30,50.8960307905446);

\path [draw=forestgreen4416044, semithick]
(axis cs:100,44.2054895805637)
--(axis cs:100,50.3945104194363);

\addplot [semithick, steelblue31119180, dotted, mark=*, mark size=3, mark options={solid}]
table {%
0.1 49.7
0.3 46.9
1 49.3
3 49.5
10 47.1
30 52.8
100 48.7
};
%\addlegendentry{DL_D3}
\addplot [semithick, darkorange25512714, dotted, mark=pentagon*, mark size=3, mark options={solid}]
table {%
0.1 47.6
0.3 50
1 48.3
3 46.8
10 50.3
30 47.8
100 47
};
%\addlegendentry{DL_D3}
\addplot [semithick, forestgreen4416044, dotted, mark=square*, mark size=3, mark options={solid}]
table {%
0.1 46.6
0.3 45.8
1 48
3 47.1
10 48.1
30 47.8
100 47.3
};
%\addlegendentry{DL_D3}
\end{axis}

\end{tikzpicture}}}
    
    \subfloat[\centering Dark Chess]{\scalebox{0.6}{% This file was created with tikzplotlib v0.10.1.
\begin{tikzpicture}

\definecolor{crimson2143940}{RGB}{214,39,40}
\definecolor{darkgray176}{RGB}{176,176,176}
\definecolor{darkorange25512714}{RGB}{255,127,14}
\definecolor{forestgreen4416044}{RGB}{44,160,44}
\definecolor{lightgray204}{RGB}{204,204,204}
\definecolor{mediumpurple148103189}{RGB}{148,103,189}
\definecolor{sienna1408675}{RGB}{140,86,75}
\definecolor{steelblue31119180}{RGB}{31,119,180}

\begin{axis}[
legend cell align={left},
legend style={
  fill opacity=0.8,
  draw opacity=1,
  text opacity=1,
  at={(0.03,0.97)},
  anchor=north west,
  draw=lightgray204
},
log basis x={10},
tick align=outside,
tick pos=left,
x grid style={darkgray176},
xlabel={Budget},
xmin=0.0630957344480193, xmax=1584.89319246111,
xmode=log,
xtick style={color=black},
xtick={0.001,0.01,0.1,1,10,100,1000,10000,100000},
xticklabels={
  \(\displaystyle {10^{-3}}\),
  \(\displaystyle {10^{-2}}\),
  \(\displaystyle {10^{-1}}\),
  \(\displaystyle {10^{0}}\),
  \(\displaystyle {10^{1}}\),
  \(\displaystyle {10^{2}}\),
  \(\displaystyle {10^{3}}\),
  \(\displaystyle {10^{4}}\)
  \(\displaystyle {10^{5}}\)
},
y grid style={darkgray176},
ylabel={Winning rate},
ytick=\empty,
ymin=0, ymax=100,
ytick style={color=black}
]
\path [draw=steelblue31119180, semithick]
(axis cs:0.1,43.9198406933081)
--(axis cs:0.1,52.6801593066919);

\path [draw=steelblue31119180, semithick]
(axis cs:1,45.7173155294957)
--(axis cs:1,54.4826844705044);

\path [draw=steelblue31119180, semithick]
(axis cs:10,40.837548960504)
--(axis cs:10,49.562451039496);

\path [draw=steelblue31119180, semithick]
(axis cs:100,41.0358937354826)
--(axis cs:100,49.7641062645174);

\path [draw=steelblue31119180, semithick]
(axis cs:1000,41.24)
--(axis cs:1000,49.96);

\path [draw=darkorange25512714, semithick]
(axis cs:0.1,46.7183675042286)
--(axis cs:0.1,55.4816324957714);

\path [draw=darkorange25512714, semithick]
(axis cs:1,57.8475768451388)
--(axis cs:1,66.3524231548612);

\path [draw=darkorange25512714, semithick]
(axis cs:10,61.0247307305995)
--(axis cs:10,69.3752692694005);

\path [draw=darkorange25512714, semithick]
(axis cs:100,64.5318406953513)
--(axis cs:100,72.6681593046487);

\path [draw=darkorange25512714, semithick]
(axis cs:1000,65.8)
--(axis cs:1000,73.8);

\path [draw=forestgreen4416044, semithick]
(axis cs:0.1,45.9173856532887)
--(axis cs:0.1,54.6826143467113);

\path [draw=forestgreen4416044, semithick]
(axis cs:1,65.568074667358)
--(axis cs:1,73.631925332642);

\path [draw=forestgreen4416044, semithick]
(axis cs:10,77.7752568283154)
--(axis cs:10,84.6247431716846);

\path [draw=forestgreen4416044, semithick]
(axis cs:100,87.9298577471276)
--(axis cs:100,93.0701422528724);

\path [draw=forestgreen4416044, semithick]
(axis cs:1000,99.37)
--(axis cs:1000,100);

\addplot [semithick, steelblue31119180, dotted, mark=*, mark size=3, mark options={solid}]
table {%
0.1 48.3
1 50.1
10 45.2
100 45.4
1000 45.6
};
% \addlegendentry{PIMC}
\addplot [semithick, darkorange25512714, dotted, mark=pentagon*, mark size=3, mark options={solid}]
table {%
0.1 51.1
1 62.1
10 65.2
100 68.6
1000 69.8
};
% \addlegendentry{EPIMC\_D2}
\addplot [semithick, forestgreen4416044, dotted, mark=square*, mark size=3, mark options={solid}]
table {%
0.1 50.3
1 69.6
10 81.2
100 90.5
1000 99.8
};
% \addlegendentry{EPIMC\_D3}
% \addplot [semithick, crimson2143940, dotted, mark=square*, mark size=3, mark options={solid}]
% table {%
% 0.1 2.1
% 1 3.7
% 10 14.05
% 100 23.7
% 1000 40
% };
% % \addlegendentry{OOS}
% \addplot [semithick, mediumpurple148103189, dotted, mark=triangle*, mark size=3, mark options={solid}]
% table {%
% 0.1 11.8
% 1 15.2
% 10 12.6
% 100 14.2
% 1000 14.5
% };
% % \addlegendentry{IIMC}
% \addplot [semithick, sienna1408675, dotted, mark=diamond*, mark size=3, mark options={solid}]
% table {%
% 0.1 27
% 1 57.3
% 10 74.9
% 100 92.2
% 1000 77.25
% };
% % \addlegendentry{IS-MCTS}
% \addplot [semithick, black, dotted, mark=halfcircle, mark size=3, mark options={solid}]
% table {%
% 0.1 3.4
% 1 3.4
% 10 3.4
% 100 3.4
% 1000 3.4
% };
%\addlegendentry{Random}
\end{axis}

\end{tikzpicture}}}
    \subfloat[\centering Phantom Tic-Tac-Toe]{\scalebox{0.6}{% This file was created with tikzplotlib v0.10.1.
\begin{tikzpicture}

\definecolor{crimson2143940}{RGB}{214,39,40}
\definecolor{darkgray176}{RGB}{176,176,176}
\definecolor{darkorange25512714}{RGB}{255,127,14}
\definecolor{forestgreen4416044}{RGB}{44,160,44}
\definecolor{lightgray204}{RGB}{204,204,204}
\definecolor{mediumpurple148103189}{RGB}{148,103,189}
\definecolor{sienna1408675}{RGB}{140,86,75}
\definecolor{steelblue31119180}{RGB}{31,119,180}

\begin{axis}[
legend cell align={left},
legend style={
  fill opacity=0.8,
  draw opacity=1,
  text opacity=1,
  at={(0.5,0.09)},
  anchor=south,
  draw=lightgray204
},
log basis x={10},
tick align=outside,
tick pos=left,
x grid style={darkgray176},
xlabel={Budget},
xmin=0.0707945784384138, xmax=141.253754462275,
xmode=log,
xtick style={color=black},
xtick={0.001,0.01,0.1,1,10,100,1000,10000},
xticklabels={
  \(\displaystyle {10^{-3}}\),
  \(\displaystyle {10^{-2}}\),
  \(\displaystyle {10^{-1}}\),
  \(\displaystyle {10^{0}}\),
  \(\displaystyle {10^{1}}\),
  \(\displaystyle {10^{2}}\),
  \(\displaystyle {10^{3}}\),
  \(\displaystyle {10^{4}}\)
  \(\displaystyle {10^{5}}\)
},
y grid style={darkgray176},
ytick=\empty,
ymin=0, ymax=100,
ytick style={color=black}
]
\path [draw=steelblue31119180, semithick]
(axis cs:0.1,42.028681465736)
--(axis cs:0.1,50.771318534264);

\path [draw=steelblue31119180, semithick]
(axis cs:1,42.028681465736)
--(axis cs:1,50.771318534264);

\path [draw=steelblue31119180, semithick]
(axis cs:10,45.5173155294956)
--(axis cs:10,54.2826844705043);

\path [draw=steelblue31119180, semithick]
(axis cs:100,47.6208143222741)
--(axis cs:100,56.3791856777259);

% \path [draw=darkorange25512714, semithick]
% (axis cs:0.1,63.7044347261947)
% --(axis cs:0.1,71.8955652738053);

% \path [draw=darkorange25512714, semithick]
% (axis cs:1,65.0496800363428)
% --(axis cs:1,73.1503199636572);

% \path [draw=darkorange25512714, semithick]
% (axis cs:10,60.8191771144905)
% --(axis cs:10,69.1808228855095);

% \path [draw=darkorange25512714, semithick]
% (axis cs:100,63.601106213623)
% --(axis cs:100,71.798893786377);

\path [draw=forestgreen4416044, semithick]
(axis cs:0.1,59.9977673134392)
--(axis cs:0.1,68.4022326865608);

\path [draw=forestgreen4416044, semithick]
(axis cs:1,76.3872884593238)
--(axis cs:1,83.4127115406762);

\path [draw=forestgreen4416044, semithick]
(axis cs:10,74.6868847303193)
--(axis cs:10,81.9131152696807);

\path [draw=forestgreen4416044, semithick]
(axis cs:100,75.8550106843603)
--(axis cs:100,82.9449893156398);

% \path [draw=crimson2143940, semithick]
% (axis cs:0.1,10.5896889642581)
% --(axis cs:0.1,14.7103110357419);

% \path [draw=crimson2143940, semithick]
% (axis cs:1,16.4734044292466)
% --(axis cs:1,21.3265955707534);

% \path [draw=crimson2143940, semithick]
% (axis cs:10,14.1993991219683)
% --(axis cs:10,18.8006008780317);

% \path [draw=crimson2143940, semithick]
% (axis cs:100,16.9967113671645)
% --(axis cs:100,21.9032886328355);

% \path [draw=mediumpurple148103189, semithick]
% (axis cs:0.1,19.311240425292)
% --(axis cs:0.1,26.688759574708);

% \path [draw=mediumpurple148103189, semithick]
% (axis cs:1,20.1618015932805)
% --(axis cs:1,27.6381984067195);

% \path [draw=mediumpurple148103189, semithick]
% (axis cs:10,23.2994916597961)
% --(axis cs:10,31.1005083402039);

% \path [draw=mediumpurple148103189, semithick]
% (axis cs:100,21.2994290997273)
% --(axis cs:100,28.9005709002727);

% \path [draw=sienna1408675, semithick]
% (axis cs:0.1,48.9268626877264)
% --(axis cs:0.1,57.6731373122737);

% \path [draw=sienna1408675, semithick]
% (axis cs:1,59.5875409328992)
% --(axis cs:1,68.0124590671008);

% \path [draw=sienna1408675, semithick]
% (axis cs:10,62.4658958155363)
% --(axis cs:10,70.7341041844637);

% \path [draw=sienna1408675, semithick]
% (axis cs:100,52.3568328201645)
% --(axis cs:100,61.0431671798355);

\addplot [semithick, steelblue31119180, dotted, mark=*, mark size=3, mark options={solid}]
table {%
0.1 46.4
1 46.4
10 49.9
100 52
};
% \addlegendentry{PIMC}
\addplot [semithick, darkorange25512714, dotted, mark=pentagon*, mark size=3, mark options={solid}]
table {%
0.1 67.8
1 69.1
10 65
100 67.7
};
% \addlegendentry{EPIMC\_D2}
\addplot [semithick, forestgreen4416044, dotted, mark=square*, mark size=3, mark options={solid}]
table {%
0.1 64.2
1 79.9
10 78.3
100 79.4
};
% \addlegendentry{EPIMC\_D3}
% \addplot [semithick, crimson2143940, dotted, mark=square*, mark size=3, mark options={solid}]
% table {%
% 0.1 12.65
% 1 18.9
% 10 16.5
% 100 19.45
% };
% % \addlegendentry{OOS}
% \addplot [semithick, mediumpurple148103189, dotted, mark=triangle*, mark size=3, mark options={solid}]
% table {%
% 0.1 23
% 1 23.9
% 10 27.2
% 100 25.1
% };
% % \addlegendentry{IIMC}
% \addplot [semithick, sienna1408675, dotted, mark=diamond*, mark size=3, mark options={solid}]
% table {%
% 0.1 53.3
% 1 63.8
% 10 66.6
% 100 56.7
% };
% % \addlegendentry{IS MCTS}
% \addplot [semithick, black, dotted, mark=halfcircle, mark size=3, mark options={solid}]
% table {%
% 0.1 12.1
% 1 12.1
% 10 12.1
% 100 12.1
% };
\end{axis}

\end{tikzpicture}}}
    \subfloat[\centering Dark Hex \hspace{20mm}~]{\scalebox{0.6}{% This file was created with tikzplotlib v0.10.1.
\begin{tikzpicture}

\definecolor{crimson2143940}{RGB}{214,39,40}
\definecolor{darkgray176}{RGB}{176,176,176}
\definecolor{darkorange25512714}{RGB}{255,127,14}
\definecolor{forestgreen4416044}{RGB}{44,160,44}
\definecolor{lightgray204}{RGB}{204,204,204}
\definecolor{mediumpurple148103189}{RGB}{148,103,189}
\definecolor{sienna1408675}{RGB}{140,86,75}
\definecolor{steelblue31119180}{RGB}{31,119,180}

\begin{axis}[
legend cell align={left},
legend style={
  fill opacity=0.8,
  draw opacity=1,
  text opacity=1,
  at={(0.03,0.97)},
  anchor=north west,
  draw=lightgray204
},
legend pos=outer north east,
log basis x={10},
tick align=outside,
tick pos=left,
x grid style={darkgray176},
xlabel={Budget},
xmin=0.0630957344480193, xmax=1584.89319246111,
xmode=log,
xtick style={color=black},
xtick={0.001,0.01,0.1,1,10,100,1000,10000,100000},
xticklabels={
  \(\displaystyle {10^{-3}}\),
  \(\displaystyle {10^{-2}}\),
  \(\displaystyle {10^{-1}}\),
  \(\displaystyle {10^{0}}\),
  \(\displaystyle {10^{1}}\),
  \(\displaystyle {10^{2}}\),
  \(\displaystyle {10^{3}}\),
  \(\displaystyle {10^{4}}\)
  \(\displaystyle {10^{5}}\)
},
y grid style={darkgray176},
ymin=0, ymax=100,
ytick=\empty,
ytick style={color=black}
]
\path [draw=steelblue31119180, semithick]
(axis cs:0.1,50.437548960504)
--(axis cs:0.1,59.162451039496);

\path [draw=steelblue31119180, semithick]
(axis cs:1,43.4215512664872)
--(axis cs:1,52.1784487335128);

\path [draw=steelblue31119180, semithick]
(axis cs:10,43.0232361763513)
--(axis cs:10,51.7767638236487);

\path [draw=steelblue31119180, semithick]
(axis cs:100,40.0448817641768)
--(axis cs:100,48.7551182358232);

\path [draw=steelblue31119180, semithick]
(axis cs:1000,43.4)
--(axis cs:1000,53);

\path [draw=darkorange25512714, semithick]
(axis cs:0.1,49.8327962777081)
--(axis cs:0.1,58.5672037222919);

\path [draw=darkorange25512714, semithick]
(axis cs:1,57.3368851049966)
--(axis cs:1,65.8631148950034);

\path [draw=darkorange25512714, semithick]
(axis cs:10,65.568074667358)
--(axis cs:10,73.631925332642);

\path [draw=darkorange25512714, semithick]
(axis cs:100,57.9497726517279)
--(axis cs:100,66.4502273482721);

\path [draw=darkorange25512714, semithick]
(axis cs:1000,42.23)
--(axis cs:1000,50.97);

\path [draw=forestgreen4416044, semithick]
(axis cs:0.1,42.4262917298933)
--(axis cs:0.1,51.1737082701067);

\path [draw=forestgreen4416044, semithick]
(axis cs:1,56.08081464656)
--(axis cs:1,63.3477567820115);

\path [draw=forestgreen4416044, semithick]
(axis cs:10,61.4360815992626)
--(axis cs:10,69.7639184007375);

\path [draw=forestgreen4416044, semithick]
(axis cs:100,80.7865562397951)
--(axis cs:100,87.2134437602049);

\path [draw=forestgreen4416044, semithick]
(axis cs:1000,88.8)
--(axis cs:1000,93.8);

\addplot [semithick, steelblue31119180, dotted, mark=*, mark size=3, mark options={solid}]
table {%
0.1 54.8
1 47.8
10 47.4
100 44.4
1000 48.2
};
% \addlegendentry{PIMC}
\addplot [semithick, darkorange25512714, dotted, mark=pentagon*, mark size=3, mark options={solid}]
table {%
0.1 54.2
1 61.6
10 69.6
100 62.2
1000 46.6
};
% \addlegendentry{EPIMC\_D2}
\addplot [semithick, forestgreen4416044, dotted, mark=square*, mark size=3, mark options={solid}]
table {%
0.1 46.8
1 59.7142857142857
10 65.6
100 84
1000 91.3
};
% \addlegendentry{EPIMC\_D3}
% \addplot [semithick, crimson2143940, dotted, mark=square*, mark size=3, mark options={solid}]
% table {%
% 0.1 0.4
% 1 0.4
% 10 0.6
% 100 0.777777777777778
% 1000 1
% };
% \addlegendentry{OOS}
% \addplot [semithick, mediumpurple148103189, dotted, mark=triangle*, mark size=3, mark options={solid}]
% table {%
% 0.1 0.2
% 1 1.8
% 10 1.8
% 100 3
% 1000 2.38095238095238
% };
% \addlegendentry{IIMC}
% \addplot [semithick, sienna1408675, dotted, mark=diamond*, mark size=3, mark options={solid}]
% table {%
% 0.1 25.4
% 1 58.2
% 10 81.8
% 100 68.6666666666667
% 1000 65.5
% };
% \addlegendentry{IS MCTS}
% \addplot [semithick, black, dotted, mark=halfcircle, mark size=3, mark options={solid}]
% table {%
% 0.1 0.5
% 1 0.5
% 10 0.5
% 100 0.5
% 1000 0.5
% };
% \addlegendentry{Random}
\end{axis}

\end{tikzpicture}}}

    \caption{\centering Winning rate of EPIMC when the depth from 1 to 3. The opponent is PIMC with one second of the budget.}

~\label{fig:E_hyper}
\end{figure*}
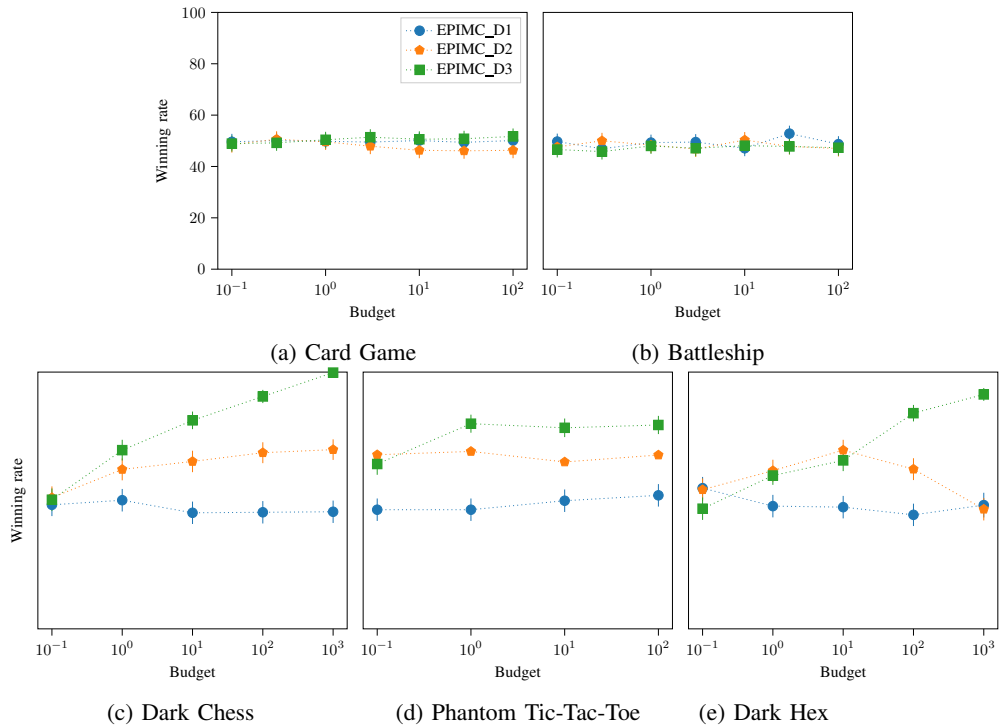
In Figure~\ref{fig:E_hyper}, we examine the performance of EPIMC according to various depths varying from $1$ to $3$. 
As expected, for Card Game and Battleship, increasing the depth does not lead to any improvement as both games have a majority of their observation public.
Conversely, for games where the observations are private, we observe an important increase in performance
At $100$ seconds for Dark Chess, we achieve $80\%$/$65\%$/$45\%$ winning rates at depths $3$/$2$/$1$. 
More than that, at $1000$ seconds, EPIMC at depth $3$ wins close to $100\%$ for both Dark Hex and Dark Chess. 
Interestingly, in Dark Hex, at $1000$ seconds, similar performances are observed at depth $2$ and $1$. 
Indeed reducing the strategy fusion does not necessarily mean that the strategy produced at the end will be changed. 

In the following, the experiments are reduced to Dark Chess, Phantom Tic-Tac-Toe, and Dark Hex, as our methods work when games have private information, yet, the extended experimentation for Battleship and Card Game are available in the Appendix. Furthermore, we replicate the experimentation of this section against IS-MCTS instead of PIMC

\paragraph{Subgame resolution}
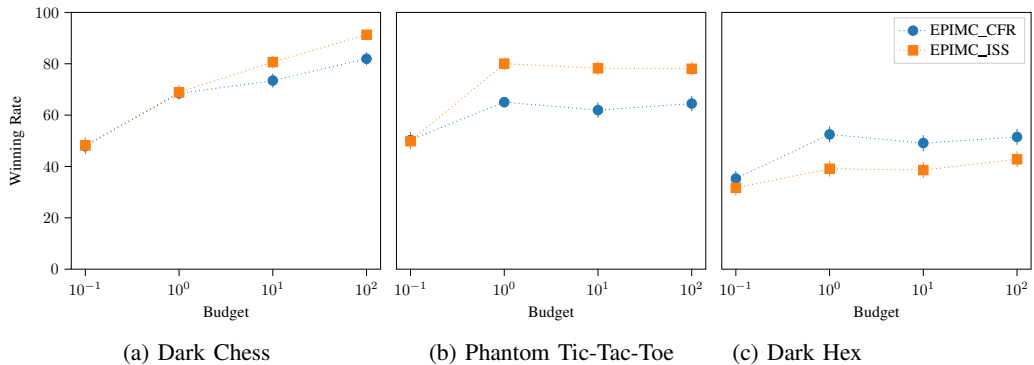
\begin{figure*}[!htbp]
\centering

    \subfloat[\centering Dark Chess]{\scalebox{0.6}{% This file was created with tikzplotlib v0.10.1.
\begin{tikzpicture}

\definecolor{darkgray176}{RGB}{176,176,176}
\definecolor{darkorange25512714}{RGB}{255,127,14}
\definecolor{lightgray204}{RGB}{204,204,204}
\definecolor{steelblue31119180}{RGB}{31,119,180}

\begin{axis}[
legend cell align={left},
legend style={
  fill opacity=0.8,
  draw opacity=1,
  text opacity=1,
  at={(0.03,0.97)},
  anchor=north west,
  draw=lightgray204
},
log basis x={10},
tick align=outside,
tick pos=left,
x grid style={darkgray176},
xlabel={Budget},
xmin=0.0707945784384138, xmax=141.253754462275,
xmode=log,
xtick style={color=black},
xtick={0.001,0.01,0.1,1,10,100,1000,10000},
xticklabels={
  \(\displaystyle {10^{-3}}\),
  \(\displaystyle {10^{-2}}\),
  \(\displaystyle {10^{-1}}\),
  \(\displaystyle {10^{0}}\),
  \(\displaystyle {10^{1}}\),
  \(\displaystyle {10^{2}}\),
  \(\displaystyle {10^{3}}\),
  \(\displaystyle {10^{4}}\)
},
y grid style={darkgray176},
ylabel={Winning Rate},
ymin=0, ymax=100,
ytick style={color=black}
]
\path [draw=steelblue31119180, semithick]
(axis cs:0.1,44.8037024458234)
--(axis cs:0.1,50.9962975541766);

\path [draw=steelblue31119180, semithick]
(axis cs:1,66.4387333820199)
--(axis cs:1,70.5112666179801);

\path [draw=steelblue31119180, semithick]
(axis cs:10,70.7129405275004)
--(axis cs:10,76.1870594724996);

\path [draw=steelblue31119180, semithick]
(axis cs:100,79.4314009674734)
--(axis cs:100,84.4574879214155);

\path [draw=darkorange25512714, semithick]
(axis cs:0.1,45.1029767169102)
--(axis cs:0.1,51.2970232830897);

\path [draw=darkorange25512714, semithick]
(axis cs:1,65.9796353743818)
--(axis cs:1,71.7203646256181);

\path [draw=darkorange25512714, semithick]
(axis cs:10,78.2015077039125)
--(axis cs:10,83.0984922960875);

\path [draw=darkorange25512714, semithick]
(axis cs:100,89.5531682118761)
--(axis cs:100,93.0468317881239);

\addplot [semithick, steelblue31119180, dotted, mark=*, mark size=3, mark options={solid}]
table {%
0.1 47.9
1 68.475
10 73.45
100 81.9444444444444
};
%\addlegendentry{EPIMC\_Minimax}
\addplot [semithick, darkorange25512714, dotted, mark=square*, mark size=3, mark options={solid}]
table {%
0.1 48.2
1 68.85
10 80.65
100 91.3
};
%\addlegendentry{EPIMC\_ISS}
\end{axis}

\end{tikzpicture}}}
    \subfloat[\centering Phantom Tic-Tac-Toe]{\scalebox{0.6}{% This file was created with tikzplotlib v0.10.1.
\begin{tikzpicture}

\definecolor{darkgray176}{RGB}{176,176,176}
\definecolor{darkorange25512714}{RGB}{255,127,14}
\definecolor{lightgray204}{RGB}{204,204,204}
\definecolor{steelblue31119180}{RGB}{31,119,180}

\begin{axis}[
legend cell align={left},
legend style={fill opacity=0.8, draw opacity=1, text opacity=1, draw=lightgray204},
log basis x={10},
tick align=outside,
tick pos=left,
x grid style={darkgray176},
xlabel={Budget},
xmin=0.0707945784384138, xmax=141.253754462275,
xmode=log,
xtick style={color=black},
xtick={0.001,0.01,0.1,1,10,100,1000,10000},
xticklabels={
  \(\displaystyle {10^{-3}}\),
  \(\displaystyle {10^{-2}}\),
  \(\displaystyle {10^{-1}}\),
  \(\displaystyle {10^{0}}\),
  \(\displaystyle {10^{1}}\),
  \(\displaystyle {10^{2}}\),
  \(\displaystyle {10^{3}}\),
  \(\displaystyle {10^{4}}\)
},
y grid style={darkgray176},
ymin=0, ymax=100,
ytick style={color=black},
ytick=\empty
]
\path [draw=steelblue31119180, semithick]
(axis cs:0.1,47.201023676115)
--(axis cs:0.1,53.398976323885);

\path [draw=steelblue31119180, semithick]
(axis cs:1,62.9602789664647)
--(axis cs:1,67.1397210335353);

\path [draw=steelblue31119180, semithick]
(axis cs:10,58.9407791845729)
--(axis cs:10,64.9592208154271);

\path [draw=steelblue31119180, semithick]
(axis cs:100,61.5341436986934)
--(axis cs:100,67.4658563013066);

\path [draw=darkorange25512714, semithick]
(axis cs:0.1,46.7509818387108)
--(axis cs:0.1,52.9490181612892);

\path [draw=darkorange25512714, semithick]
(axis cs:1,77.5731016177485)
--(axis cs:1,82.5268983822515);

\path [draw=darkorange25512714, semithick]
(axis cs:10,75.6930167970829)
--(axis cs:10,80.8069832029171);

\path [draw=darkorange25512714, semithick]
(axis cs:100,75.4845722547692)
--(axis cs:100,80.6154277452308);

\addplot [semithick, steelblue31119180, dotted, mark=*, mark size=3, mark options={solid}]
table {%
0.1 50.3
1 65.05
10 61.95
100 64.5
};
%\addlegendentry{EPIMC\_CFR}
\addplot [semithick, darkorange25512714, dotted, mark=square*, mark size=3, mark options={solid}]
table {%
0.1 49.85
1 80.05
10 78.25
100 78.05
};
%\addlegendentry{EPIMC\_ISS}
\end{axis}

\end{tikzpicture}}}
    \subfloat[\centering Dark Hex \hspace{20mm}~]{\scalebox{0.6}{% This file was created with tikzplotlib v0.10.1.
\begin{tikzpicture}

\definecolor{darkgray176}{RGB}{176,176,176}
\definecolor{darkorange25512714}{RGB}{255,127,14}
\definecolor{lightgray204}{RGB}{204,204,204}
\definecolor{steelblue31119180}{RGB}{31,119,180}

\begin{axis}[
legend cell align={left},
legend style={fill opacity=0.8, draw opacity=1, text opacity=1, draw=lightgray204},
log basis x={10},
tick align=outside,
tick pos=left,
% legend pos=outer north east,
%title={DarkHex\_CFR},
x grid style={darkgray176},
xlabel={Budget},
xmin=0.0707945784384138, xmax=141.253754462275,
xmode=log,
xtick style={color=black},
xtick={0.001,0.01,0.1,1,10,100,1000,10000},
xticklabels={
  \(\displaystyle {10^{-3}}\),
  \(\displaystyle {10^{-2}}\),
  \(\displaystyle {10^{-1}}\),
  \(\displaystyle {10^{0}}\),
  \(\displaystyle {10^{1}}\),
  \(\displaystyle {10^{2}}\),
  \(\displaystyle {10^{3}}\),
  \(\displaystyle {10^{4}}\)
},
y grid style={darkgray176},
ymin=0, ymax=100,
ytick style={color=black},
ytick=\empty
]
\path [draw=steelblue31119180, semithick]
(axis cs:0.1,32.3379283168701)
--(axis cs:0.1,38.2620716831299);

\path [draw=steelblue31119180, semithick]
(axis cs:1,49.4048441073187)
--(axis cs:1,55.5951558926813);

\path [draw=steelblue31119180, semithick]
(axis cs:10,46.0014699769084)
--(axis cs:10,52.1985300230916);

\path [draw=steelblue31119180, semithick]
(axis cs:100,48.4023627714014)
--(axis cs:100,54.5976372285986);

\path [draw=darkorange25512714, semithick]
(axis cs:0.1,28.815994699034)
--(axis cs:0.1,34.584005300966);

\path [draw=darkorange25512714, semithick]
(axis cs:1,36.0755034296597)
--(axis cs:1,42.1244965703403);

\path [draw=darkorange25512714, semithick]
(axis cs:10,35.5825928905764)
--(axis cs:10,41.6174071094236);

\path [draw=darkorange25512714, semithick]
(axis cs:100,39.7332669734716)
--(axis cs:100,45.8667330265284);

\addplot [semithick, steelblue31119180, dotted, mark=*, mark size=3, mark options={solid}]
table {%
0.1 35.3
1 52.5
10 49.1
100 51.5
};
\addlegendentry{EPIMC\_CFR}
\addplot [semithick, darkorange25512714, dotted, mark=square*, mark size=3, mark options={solid}]
table {%
0.1 31.7
1 39.1
10 38.6
100 42.8
};
\addlegendentry{EPIMC\_ISS}
\end{axis}

\end{tikzpicture}}}

    \caption{\centering Winning rate of EPIMC when the subgame is CFR+ or ISS. The opponent is PIMC with one second of the budget.}

~\label{fig:_subgame}
\end{figure*}

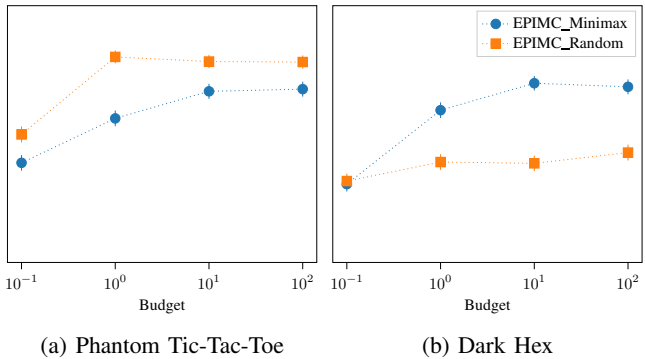
\begin{figure}[!htbp]
    \centering
    
        %\subfloat[\centering Dark Chess]{\scalebox{0.7}{\input{Image/DarkChess}}}
        \subfloat[\centering Phantom Tic-Tac-Toe]{\scalebox{0.6}{% This file was created with tikzplotlib v0.10.1.
\begin{tikzpicture}

\definecolor{darkgray176}{RGB}{176,176,176}
\definecolor{darkorange25512714}{RGB}{255,127,14}
\definecolor{lightgray204}{RGB}{204,204,204}
\definecolor{steelblue31119180}{RGB}{31,119,180}

\begin{axis}[
legend cell align={left},
legend style={fill opacity=0.8, draw opacity=1, text opacity=1, draw=lightgray204},
log basis x={10},
tick align=outside,
tick pos=left,
x grid style={darkgray176},
xlabel={Budget},
xmin=0.0707945784384138, xmax=141.253754462275,
xmode=log,
xtick style={color=black},
xtick={0.001,0.01,0.1,1,10,100,1000,10000},
xticklabels={
  \(\displaystyle {10^{-3}}\),
  \(\displaystyle {10^{-2}}\),
  \(\displaystyle {10^{-1}}\),
  \(\displaystyle {10^{0}}\),
  \(\displaystyle {10^{1}}\),
  \(\displaystyle {10^{2}}\),
  \(\displaystyle {10^{3}}\),
  \(\displaystyle {10^{4}}\)
},
y grid style={darkgray176},
ytick=\empty,
ymin=0, ymax=100,
% ytick style={color=black}
]
\path [draw=steelblue31119180, semithick]
(axis cs:0.1,35.7797169510127)
--(axis cs:0.1,41.8202830489873);

\path [draw=steelblue31119180, semithick]
(axis cs:1,53.0241173520435)
--(axis cs:1,59.1758826479565);

\path [draw=steelblue31119180, semithick]
(axis cs:10,63.727839490377)
--(axis cs:10,69.572160509623);

\path [draw=steelblue31119180, semithick]
(axis cs:100,64.5969826042547)
--(axis cs:100,70.4030173957453);

\path [draw=darkorange25512714, semithick]
(axis cs:0.1,46.7509818387108)
--(axis cs:0.1,52.9490181612892);

\path [draw=darkorange25512714, semithick]
(axis cs:1,77.5731016177485)
--(axis cs:1,82.5268983822515);

\path [draw=darkorange25512714, semithick]
(axis cs:10,75.6930167970829)
--(axis cs:10,80.8069832029171);

\path [draw=darkorange25512714, semithick]
(axis cs:100,75.4845722547692)
--(axis cs:100,80.6154277452308);

\addplot [semithick, steelblue31119180, dotted, mark=*, mark size=3, mark options={solid}]
table {%
0.1 38.8
1 56.1
10 66.65
100 67.5
};
%\addlegendentry{EPIMC\_Minimax}
\addplot [semithick, darkorange25512714, dotted, mark=square*, mark size=3, mark options={solid}]
table {%
0.1 49.85
1 80.05
10 78.25
100 78.05
};
%\addlegendentry{EPIMC\_Random}
\end{axis}

\end{tikzpicture}}}
        \subfloat[\centering Dark Hex]{\scalebox{0.6}{% This file was created with tikzplotlib v0.10.1.
\begin{tikzpicture}

\definecolor{darkgray176}{RGB}{176,176,176}
\definecolor{darkorange25512714}{RGB}{255,127,14}
\definecolor{lightgray204}{RGB}{204,204,204}
\definecolor{steelblue31119180}{RGB}{31,119,180}

\begin{axis}[
legend cell align={left},
legend style={fill opacity=0.8, draw opacity=1, text opacity=1, draw=lightgray204},
log basis x={10},
tick align=outside,
tick pos=left,
% legend pos=outer north east,
%title={DarkHex\_MaxN},
x grid style={darkgray176},
xlabel={Budget},
xmin=0.0707945784384138, xmax=141.253754462275,
xmode=log,
xtick style={color=black},
xtick={0.001,0.01,0.1,1,10,100,1000,10000},
xticklabels={
  \(\displaystyle {10^{-3}}\),
  \(\displaystyle {10^{-2}}\),
  \(\displaystyle {10^{-1}}\),
  \(\displaystyle {10^{0}}\),
  \(\displaystyle {10^{1}}\),
  \(\displaystyle {10^{2}}\),
  \(\displaystyle {10^{3}}\),
  \(\displaystyle {10^{4}}\)
},
y grid style={darkgray176},
ytick=\empty,
ymin=0, ymax=100,
ytick style={color=black}
% ytick style={color=black}
]
\path [draw=steelblue31119180, semithick]
(axis cs:0.1,27.6463651950539)
--(axis cs:0.1,33.3536348049461);

\path [draw=steelblue31119180, semithick]
(axis cs:1,56.2550467957619)
--(axis cs:1,62.3449532042381);

\path [draw=steelblue31119180, semithick]
(axis cs:10,66.95431218578)
--(axis cs:10,72.64568781422);

\path [draw=steelblue31119180, semithick]
(axis cs:100,65.5184400224878)
--(axis cs:100,71.2815599775122);

\path [draw=darkorange25512714, semithick]
(axis cs:0.1,28.815994699034)
--(axis cs:0.1,34.584005300966);

\path [draw=darkorange25512714, semithick]
(axis cs:1,36.0755034296597)
--(axis cs:1,42.1244965703403);

\path [draw=darkorange25512714, semithick]
(axis cs:10,35.5825928905764)
--(axis cs:10,41.6174071094236);

\path [draw=darkorange25512714, semithick]
(axis cs:100,39.7332669734716)
--(axis cs:100,45.8667330265284);

\addplot [semithick, steelblue31119180, dotted, mark=*, mark size=3, mark options={solid}]
table {%
0.1 30.5
1 59.3
10 69.8
100 68.4
};
\addlegendentry{EPIMC\_Minimax}
\addplot [semithick, darkorange25512714, dotted, mark=square*, mark size=3, mark options={solid}]
table {%
0.1 31.7
1 39.1
10 38.6
100 42.8
};
\addlegendentry{EPIMC\_Random}
\end{axis}

\end{tikzpicture}}}
        \caption{\centering Winning rate of EPIMC when the perfect information leaf evaluator is Minimax or Random Rollout. The opponent is PIMC with one second of the budget.}
    ~\label{fig:_leaf}
\end{figure}

In Figure~\ref{fig:_subgame}, we compare CFR+ against Information Set Search (ISS) in the Extended game resolution. 
CFR+ is used with $1000$ iterations. 
Using CFR+ in Dark Hex $4\times 4$ was too costly in computational time, and using it on Dark Hex $3\times 3$.
As can be observed, neither method is superior to the other, worse performances are obtained with CFR+ in Dark Chess and Phantom Tic-Tac-Toe but better performance is achieved in Dark Hex. 

\paragraph{Perfect information leaf evaluator}

In Figure~\ref{fig:_leaf}, we compare  Minimax with alpha-beta pruning against random rollout for the perfect information leaf evaluator. 
Due to the important cost of using Minimax, the test was not conducted in Dark Chess, and as before, the size Dark Hex was reduced to $3\times 3$.
As before, neither method is superior to the other, as worse performance is obtained with Minimax in Phantom Tic-Tac-Toe but stronger performance is achieved in Dark Hex. 
However, the differences between the two methods are important in Dark Hex, where Minimax obtains performance close to $70\%$ while Random Rollout obtains performance close to $50\%$. 

\paragraph{Against other online algorithms}
\begin{figure*}[!htbp]
\centering
    \subfloat[\centering Dark Chess]{\scalebox{0.6}{% This file was created with tikzplotlib v0.10.1.
\begin{tikzpicture}

\definecolor{crimson2143940}{RGB}{214,39,40}
\definecolor{darkgray176}{RGB}{176,176,176}
\definecolor{darkorange25512714}{RGB}{255,127,14}
\definecolor{forestgreen4416044}{RGB}{44,160,44}
\definecolor{lightgray204}{RGB}{204,204,204}
\definecolor{mediumpurple148103189}{RGB}{148,103,189}
\definecolor{sienna1408675}{RGB}{140,86,75}
\definecolor{steelblue31119180}{RGB}{31,119,180}

\begin{axis}[
legend cell align={left},
legend style={
  fill opacity=0.8,
  draw opacity=1,
  text opacity=1,
  at={(0.03,0.97)},
  anchor=north west,
  draw=lightgray204
},
log basis x={10},
tick align=outside,
tick pos=left,
x grid style={darkgray176},
xlabel={Budget},
xmin=0.0630957344480193, xmax=1584.89319246111,
xmode=log,
xtick style={color=black},
xtick={0.001,0.01,0.1,1,10,100,1000,10000,100000},
xticklabels={
  \(\displaystyle {10^{-3}}\),
  \(\displaystyle {10^{-2}}\),
  \(\displaystyle {10^{-1}}\),
  \(\displaystyle {10^{0}}\),
  \(\displaystyle {10^{1}}\),
  \(\displaystyle {10^{2}}\),
  \(\displaystyle {10^{3}}\),
  \(\displaystyle {10^{4}}\)
  \(\displaystyle {10^{5}}\)
},
y grid style={darkgray176},
ylabel={Winning rate},
ymin=0, ymax=100,
ytick style={color=black}
]
\path [draw=steelblue31119180, semithick]
(axis cs:0.1,43.9198406933081)
--(axis cs:0.1,52.6801593066919);

\path [draw=steelblue31119180, semithick]
(axis cs:1,45.7173155294957)
--(axis cs:1,54.4826844705044);

\path [draw=steelblue31119180, semithick]
(axis cs:10,40.837548960504)
--(axis cs:10,49.562451039496);

\path [draw=steelblue31119180, semithick]
(axis cs:100,41.0358937354826)
--(axis cs:100,49.7641062645174);

\path [draw=steelblue31119180, semithick]
(axis cs:1000,41.24)
--(axis cs:1000,49.96);

\path [draw=darkorange25512714, semithick]
(axis cs:0.1,46.7183675042286)
--(axis cs:0.1,55.4816324957714);

\path [draw=darkorange25512714, semithick]
(axis cs:1,57.8475768451388)
--(axis cs:1,66.3524231548612);

\path [draw=darkorange25512714, semithick]
(axis cs:10,61.0247307305995)
--(axis cs:10,69.3752692694005);

\path [draw=darkorange25512714, semithick]
(axis cs:100,64.5318406953513)
--(axis cs:100,72.6681593046487);

\path [draw=darkorange25512714, semithick]
(axis cs:1000,65.8)
--(axis cs:1000,73.8);

\path [draw=forestgreen4416044, semithick]
(axis cs:0.1,45.9173856532887)
--(axis cs:0.1,54.6826143467113);

\path [draw=forestgreen4416044, semithick]
(axis cs:1,65.568074667358)
--(axis cs:1,73.631925332642);

\path [draw=forestgreen4416044, semithick]
(axis cs:10,77.7752568283154)
--(axis cs:10,84.6247431716846);

\path [draw=forestgreen4416044, semithick]
(axis cs:100,87.9298577471276)
--(axis cs:100,93.0701422528724);

\path [draw=forestgreen4416044, semithick]
(axis cs:1000,99.4)
--(axis cs:1000,100);

\path [draw=crimson2143940, semithick]
(axis cs:0.1,1.21129614381393)
--(axis cs:0.1,2.98870385618608);

\path [draw=crimson2143940, semithick]
(axis cs:1,2.53004252384969)
--(axis cs:1,4.86995747615031);

\path [draw=crimson2143940, semithick]
(axis cs:10,12.5269920295678)
--(axis cs:10,15.5730079704322);

\path [draw=crimson2143940, semithick]
(axis cs:100,21.0643210180297)
--(axis cs:100,26.3356789819703);

\path [draw=crimson2143940, semithick]
(axis cs:1000,35.9)
--(axis cs:1000,44.5);

\path [draw=mediumpurple148103189, semithick]
(axis cs:0.1,8.97221513689602)
--(axis cs:0.1,14.627784863104);

\path [draw=mediumpurple148103189, semithick]
(axis cs:1,12.0530431410647)
--(axis cs:1,18.3469568589353);

\path [draw=mediumpurple148103189, semithick]
(axis cs:10,9.691212079233)
--(axis cs:10,15.508787920767);

\path [draw=mediumpurple148103189, semithick]
(axis cs:100,11.1404406278028)
--(axis cs:100,17.2595593721972);

\path [draw=mediumpurple148103189, semithick]
(axis cs:1000,11.2)
--(axis cs:1000,17.2);

\path [draw=sienna1408675, semithick]
(axis cs:0.1,23.1085237762515)
--(axis cs:0.1,30.8914762237485);

\path [draw=sienna1408675, semithick]
(axis cs:1,52.9642691190527)
--(axis cs:1,61.6357308809473);

\path [draw=sienna1408675, semithick]
(axis cs:10,71.0994290997273)
--(axis cs:10,78.7005709002728);

\path [draw=sienna1408675, semithick]
(axis cs:100,89.8493724004003)
--(axis cs:100,94.5506275995997);

\path [draw=sienna1408675, semithick]
(axis cs:1000,72.05)
--(axis cs:1000,79.55);

\addplot [semithick, steelblue31119180, dotted, mark=*, mark size=3, mark options={solid}]
table {%
0.1 48.3
1 50.1
10 45.2
100 45.4
1000 45.6
};
% \addlegendentry{PIMC}
\addplot [semithick, darkorange25512714, dotted, mark=*, mark size=3, mark options={solid}]
table {%
0.1 51.1
1 62.1
10 65.2
100 68.6
1000 69.8
};
% \addlegendentry{EPIMC\_D2}
\addplot [semithick, forestgreen4416044, dotted, mark=pentagon*, mark size=3, mark options={solid}]
table {%
0.1 50.3
1 69.6
10 81.2
100 90.5
1000 99.8
};
% \addlegendentry{EPIMC\_D3}
\addplot [semithick, crimson2143940, dotted, mark=square*, mark size=3, mark options={solid}]
table {%
0.1 2.1
1 3.7
10 14.05
100 23.7
1000 40.2
};
% \addlegendentry{OOS}
\addplot [semithick, mediumpurple148103189, dotted, mark=triangle*, mark size=3, mark options={solid}]
table {%
0.1 11.8
1 15.2
10 12.6
100 14.2
1000 14.2
};
% \addlegendentry{IIMC}
\addplot [semithick, sienna1408675, dotted, mark=diamond*, mark size=3, mark options={solid}]
table {%
0.1 27
1 57.3
10 74.9
100 92.2
1000 75.8
};
% \addlegendentry{IS-MCTS}
\addplot [semithick, black, dotted, mark=halfcircle, mark size=3, mark options={solid}]
table {%
0.1 3.4
1 3.4
10 3.4
100 3.4
1000 3.4
};
%\addlegendentry{Random}
\end{axis}

\end{tikzpicture}}}
    \subfloat[\centering Phantom Tic-Tac-Toe]{\scalebox{0.6}{% This file was created with tikzplotlib v0.10.1.
\begin{tikzpicture}

\definecolor{crimson2143940}{RGB}{214,39,40}
\definecolor{darkgray176}{RGB}{176,176,176}
\definecolor{darkorange25512714}{RGB}{255,127,14}
\definecolor{forestgreen4416044}{RGB}{44,160,44}
\definecolor{lightgray204}{RGB}{204,204,204}
\definecolor{mediumpurple148103189}{RGB}{148,103,189}
\definecolor{sienna1408675}{RGB}{140,86,75}
\definecolor{steelblue31119180}{RGB}{31,119,180}

\begin{axis}[
legend cell align={left},
legend style={
  fill opacity=0.8,
  draw opacity=1,
  text opacity=1,
  at={(0.5,0.09)},
  anchor=south,
  draw=lightgray204
},
log basis x={10},
tick align=outside,
tick pos=left,
x grid style={darkgray176},
xlabel={Budget},
xmin=0.0707945784384138, xmax=141.253754462275,
xmode=log,
xtick style={color=black},
xtick={0.001,0.01,0.1,1,10,100,1000,10000},
xticklabels={
  \(\displaystyle {10^{-3}}\),
  \(\displaystyle {10^{-2}}\),
  \(\displaystyle {10^{-1}}\),
  \(\displaystyle {10^{0}}\),
  \(\displaystyle {10^{1}}\),
  \(\displaystyle {10^{2}}\),
  \(\displaystyle {10^{3}}\),
  \(\displaystyle {10^{4}}\)
  \(\displaystyle {10^{5}}\)
},
y grid style={darkgray176},
ytick=\empty,
ymin=0, ymax=100,
ytick style={color=black}
]
\path [draw=steelblue31119180, semithick]
(axis cs:0.1,42.028681465736)
--(axis cs:0.1,50.771318534264);

\path [draw=steelblue31119180, semithick]
(axis cs:1,42.028681465736)
--(axis cs:1,50.771318534264);

\path [draw=steelblue31119180, semithick]
(axis cs:10,45.5173155294956)
--(axis cs:10,54.2826844705043);

\path [draw=steelblue31119180, semithick]
(axis cs:100,47.6208143222741)
--(axis cs:100,56.3791856777259);

\path [draw=darkorange25512714, semithick]
(axis cs:0.1,63.7044347261947)
--(axis cs:0.1,71.8955652738053);

\path [draw=darkorange25512714, semithick]
(axis cs:1,65.0496800363428)
--(axis cs:1,73.1503199636572);

\path [draw=darkorange25512714, semithick]
(axis cs:10,60.8191771144905)
--(axis cs:10,69.1808228855095);

\path [draw=darkorange25512714, semithick]
(axis cs:100,63.601106213623)
--(axis cs:100,71.798893786377);

\path [draw=forestgreen4416044, semithick]
(axis cs:0.1,59.9977673134392)
--(axis cs:0.1,68.4022326865608);

\path [draw=forestgreen4416044, semithick]
(axis cs:1,76.3872884593238)
--(axis cs:1,83.4127115406762);

\path [draw=forestgreen4416044, semithick]
(axis cs:10,74.6868847303193)
--(axis cs:10,81.9131152696807);

\path [draw=forestgreen4416044, semithick]
(axis cs:100,75.8550106843603)
--(axis cs:100,82.9449893156398);

\path [draw=crimson2143940, semithick]
(axis cs:0.1,10.5896889642581)
--(axis cs:0.1,14.7103110357419);

\path [draw=crimson2143940, semithick]
(axis cs:1,16.4734044292466)
--(axis cs:1,21.3265955707534);

\path [draw=crimson2143940, semithick]
(axis cs:10,14.1993991219683)
--(axis cs:10,18.8006008780317);

\path [draw=crimson2143940, semithick]
(axis cs:100,16.9967113671645)
--(axis cs:100,21.9032886328355);

\path [draw=mediumpurple148103189, semithick]
(axis cs:0.1,19.311240425292)
--(axis cs:0.1,26.688759574708);

\path [draw=mediumpurple148103189, semithick]
(axis cs:1,20.1618015932805)
--(axis cs:1,27.6381984067195);

\path [draw=mediumpurple148103189, semithick]
(axis cs:10,23.2994916597961)
--(axis cs:10,31.1005083402039);

\path [draw=mediumpurple148103189, semithick]
(axis cs:100,21.2994290997273)
--(axis cs:100,28.9005709002727);

\path [draw=sienna1408675, semithick]
(axis cs:0.1,48.9268626877264)
--(axis cs:0.1,57.6731373122737);

\path [draw=sienna1408675, semithick]
(axis cs:1,59.5875409328992)
--(axis cs:1,68.0124590671008);

\path [draw=sienna1408675, semithick]
(axis cs:10,62.4658958155363)
--(axis cs:10,70.7341041844637);

\path [draw=sienna1408675, semithick]
(axis cs:100,52.3568328201645)
--(axis cs:100,61.0431671798355);

\addplot [semithick, steelblue31119180, dotted, mark=*, mark size=3, mark options={solid}]
table {%
0.1 46.4
1 46.4
10 49.9
100 52
};
% \addlegendentry{PIMC}
\addplot [semithick, darkorange25512714, dotted, mark=*, mark size=3, mark options={solid}]
table {%
0.1 67.8
1 69.1
10 65
100 67.7
};
% \addlegendentry{EPIMC\_D2}
\addplot [semithick, forestgreen4416044, dotted, mark=pentagon*, mark size=3, mark options={solid}]
table {%
0.1 64.2
1 79.9
10 78.3
100 79.4
};
% \addlegendentry{EPIMC\_D3}
\addplot [semithick, crimson2143940, dotted, mark=square*, mark size=3, mark options={solid}]
table {%
0.1 12.65
1 18.9
10 16.5
100 19.45
};
% \addlegendentry{OOS}
\addplot [semithick, mediumpurple148103189, dotted, mark=triangle*, mark size=3, mark options={solid}]
table {%
0.1 23
1 23.9
10 27.2
100 25.1
};
% \addlegendentry{IIMC}
\addplot [semithick, sienna1408675, dotted, mark=diamond*, mark size=3, mark options={solid}]
table {%
0.1 53.3
1 63.8
10 66.6
100 56.7
};
% \addlegendentry{IS MCTS}
\addplot [semithick, black, dotted, mark=halfcircle, mark size=3, mark options={solid}]
table {%
0.1 12.1
1 12.1
10 12.1
100 12.1
};
\end{axis}

\end{tikzpicture}}}
    \subfloat[\centering Dark Hex \hspace{20mm}~]{\scalebox{0.6}{% This file was created with tikzplotlib v0.10.1.
\begin{tikzpicture}

\definecolor{crimson2143940}{RGB}{214,39,40}
\definecolor{darkgray176}{RGB}{176,176,176}
\definecolor{darkorange25512714}{RGB}{255,127,14}
\definecolor{forestgreen4416044}{RGB}{44,160,44}
\definecolor{lightgray204}{RGB}{204,204,204}
\definecolor{mediumpurple148103189}{RGB}{148,103,189}
\definecolor{sienna1408675}{RGB}{140,86,75}
\definecolor{steelblue31119180}{RGB}{31,119,180}

\begin{axis}[
legend cell align={left},
legend style={
  fill opacity=0.8,
  draw opacity=1,
  text opacity=1,
  at={(0.03,0.97)},
  anchor=north west,
  draw=lightgray204
},
legend pos=outer north east,
log basis x={10},
tick align=outside,
tick pos=left,
x grid style={darkgray176},
xlabel={Budget},
xmin=0.0630957344480193, xmax=1584.89319246111,
xmode=log,
xtick style={color=black},
xtick={0.001,0.01,0.1,1,10,100,1000,10000,100000},
xticklabels={
  \(\displaystyle {10^{-3}}\),
  \(\displaystyle {10^{-2}}\),
  \(\displaystyle {10^{-1}}\),
  \(\displaystyle {10^{0}}\),
  \(\displaystyle {10^{1}}\),
  \(\displaystyle {10^{2}}\),
  \(\displaystyle {10^{3}}\),
  \(\displaystyle {10^{4}}\)
  \(\displaystyle {10^{5}}\)
},
y grid style={darkgray176},
ymin=0, ymax=100,
ytick=\empty,
ytick style={color=black}
]
\path [draw=steelblue31119180, semithick]
(axis cs:0.1,50.437548960504)
--(axis cs:0.1,59.162451039496);

\path [draw=steelblue31119180, semithick]
(axis cs:1,43.4215512664872)
--(axis cs:1,52.1784487335128);

\path [draw=steelblue31119180, semithick]
(axis cs:10,43.0232361763513)
--(axis cs:10,51.7767638236487);

\path [draw=steelblue31119180, semithick]
(axis cs:100,40.0448817641768)
--(axis cs:100,48.7551182358232);

\path [draw=steelblue31119180, semithick]
(axis cs:1000,43.4)
--(axis cs:1000,53);

\path [draw=darkorange25512714, semithick]
(axis cs:0.1,49.8327962777081)
--(axis cs:0.1,58.5672037222919);

\path [draw=darkorange25512714, semithick]
(axis cs:1,57.3368851049966)
--(axis cs:1,65.8631148950034);

\path [draw=darkorange25512714, semithick]
(axis cs:10,65.568074667358)
--(axis cs:10,73.631925332642);

\path [draw=darkorange25512714, semithick]
(axis cs:100,57.9497726517279)
--(axis cs:100,66.4502273482721);

\path [draw=darkorange25512714, semithick]
(axis cs:1000,42.23)
--(axis cs:1000,50.97);

\path [draw=forestgreen4416044, semithick]
(axis cs:0.1,42.4262917298933)
--(axis cs:0.1,51.1737082701067);

\path [draw=forestgreen4416044, semithick]
(axis cs:1,56.08081464656)
--(axis cs:1,63.3477567820115);

\path [draw=forestgreen4416044, semithick]
(axis cs:10,61.4360815992626)
--(axis cs:10,69.7639184007375);

\path [draw=forestgreen4416044, semithick]
(axis cs:100,80.7865562397951)
--(axis cs:100,87.2134437602049);

\path [draw=forestgreen4416044, semithick]
(axis cs:1000,88.8)
--(axis cs:1000,93.8);

\path [draw=crimson2143940, semithick]
(axis cs:0.1,0.00878478557193102)
--(axis cs:0.1,0.791215214428069);

\path [draw=crimson2143940, semithick]
(axis cs:1,0.00878478557193102)
--(axis cs:1,0.791215214428069);

\path [draw=crimson2143940, semithick]
(axis cs:10,0.121342477338964)
--(axis cs:10,1.07865752266104);

\path [draw=crimson2143940, semithick]
(axis cs:100,0.203837033404174)
--(axis cs:100,1.35171852215138);

\path [draw=crimson2143940, semithick]
(axis cs:1000,0.02)
--(axis cs:1000,1.58);

\path [draw=mediumpurple148103189, semithick]
(axis cs:0.1,-0.191607803803755)
--(axis cs:0.1,0.591607803803755);

\path [draw=mediumpurple148103189, semithick]
(axis cs:1,0.634632061535928)
--(axis cs:1,2.96536793846407);

\path [draw=mediumpurple148103189, semithick]
(axis cs:10,0.634632061535928)
--(axis cs:10,2.96536793846407);

\path [draw=mediumpurple148103189, semithick]
(axis cs:100,1.50473707997557)
--(axis cs:100,4.49526292002443);

\path [draw=mediumpurple148103189, semithick]
(axis cs:1000,0.66)
--(axis cs:1000,2.86);

\path [draw=sienna1408675, semithick]
(axis cs:0.1,21.5844483114496)
--(axis cs:0.1,29.2155516885504);

\path [draw=sienna1408675, semithick]
(axis cs:1,53.8766469457145)
--(axis cs:1,62.5233530542855);

\path [draw=sienna1408675, semithick]
(axis cs:10,78.4179235916378)
--(axis cs:10,85.1820764083622);

\path [draw=sienna1408675, semithick]
(axis cs:100,64.9551045916049)
--(axis cs:100,72.3782287417285);

\path [draw=sienna1408675, semithick]
(axis cs:1000,60.82)
--(axis cs:1000,69.18);

\addplot [semithick, steelblue31119180, dotted, mark=*, mark size=3, mark options={solid}]
table {%
0.1 54.8
1 47.8
10 47.4
100 44.4
1000 48.2
};
\addlegendentry{PIMC}
\addplot [semithick, darkorange25512714, dotted, mark=*, mark size=3, mark options={solid}]
table {%
0.1 54.2
1 61.6
10 69.6
100 62.2
1000 46.6
};
\addlegendentry{EPIMC\_D2}
\addplot [semithick, forestgreen4416044, dotted, mark=pentagon*, mark size=3, mark options={solid}]
table {%
0.1 46.8
1 59.7142857142857
10 65.6
100 84
1000 91.3
};
\addlegendentry{EPIMC\_D3}
\addplot [semithick, crimson2143940, dotted, mark=square*, mark size=3, mark options={solid}]
table {%
0.1 0.4
1 0.4
10 0.6
100 0.777777777777778
1000 0.8
};
\addlegendentry{OOS}
\addplot [semithick, mediumpurple148103189, dotted, mark=triangle*, mark size=3, mark options={solid}]
table {%
0.1 0.2
1 1.8
10 1.8
100 3
1000 1.76
};
\addlegendentry{IIMC}
\addplot [semithick, sienna1408675, dotted, mark=diamond*, mark size=3, mark options={solid}]
table {%
0.1 25.4
1 58.2
10 81.8
100 68.6666666666667
1000 65
};
\addlegendentry{IS MCTS}
\addplot [semithick, black, dotted, mark=halfcircle, mark size=3, mark options={solid}]
table {%
0.1 0.5
1 0.5
10 0.5
100 0.5
1000 0.5
};
\addlegendentry{Random}
\end{axis}

\end{tikzpicture}}}

    \caption{\centering Winning rate of online algorithms. The opponent is PIMC with one second of the budget.}

~\label{fig:vsPIMC}
\end{figure*}
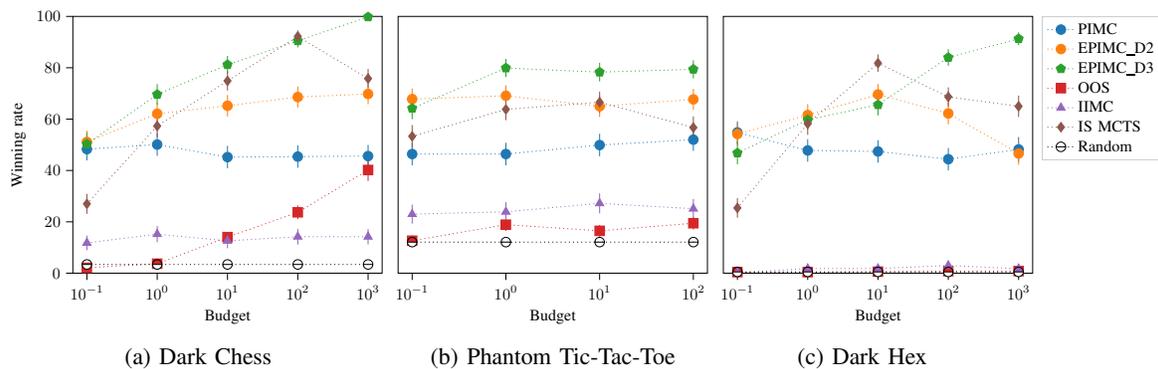

In Figure~\ref{fig:vsPIMC}, a comparative analysis between EPIMC and various online algorithms against opponents is presented. 
It is evident from the results that EPIMC and IS-MCTS exhibit superior performance compared to other algorithms. 
Notably, EPIMC achieved remarkable success, particularly at depth $3$, outperforming IS-MCTS across all benchmarks considered. 
Furthermore, it is noteworthy that even at a depth of $2$, EPIMC demonstrated comparable or even superior performance to IS-MCTS in games like Dark Chess and Phantom Tic-Tac-Toe.

\subsection{Discussion}

As we have seen and as was expected, increasing the depth has a significant impact on games with a private observation (Dark Hex / Dark Chess / Phantom Tic-Tac-Toe) and much less impact on games without much private information (Card Game / Battleship). 
On the domains tested and given a sufficient budget, augmenting the depth consistently outperforms the basic version. 
In addition, a depth of $2$ is sufficient, in most cases, to beat the state-of-the-art online algorithms. 
With regard to the other hyperparameters, we recommend using Information Set Search and Random Rollout as they require less computing time and therefore can be computed on larger games. 
However, customization could be beneficial in order to increase performance or the need to maintain theoretical properties in the subgame. 
Especially, using algorithms such as CFR+ allows to obtain the properties of CFR+ (in the subgame) such as the convergence towards the Nash equilibrium in two players.

\section{Related Work}
~\label{sec:RelatedWork}

In determinization-based algorithms, both IIMC~\cite{furtak_recursive_2013} and IS-MCTS~\cite{cowling_information_2012} extend reasoning beyond a depth of one. 
IS-MCTS increases depth with the budget, and IIMC recursively calls PIMC until the game's end. 
In both cases, as the depth exceeds PIMC, it is expected that strategy fusion will be reduced. 
However, there is no guarantee that it can be entirely eliminated, and it is challenging to quantify the extent of its reduction. 
This is because the two methods do not uniformly explore the state space, and worse, they use exploration/exploitation mechanisms based on estimates that are distorted by the presence of strategy fusion.

PIMC and our work share similarities with Unsafe/Safe Subgame Solving~\cite{burch_solving_2014}. 
However, unlike Unsafe/Safe algorithms, our method does not require an expensive pre-computed value function. 
Unsafe/Safe algorithms often solve an abstraction of the game that includes the subgame with leaf estimation and refine this during play. 
Moreover, their methods rarely detail subgame construction during the game, unlike our determinization method, which uses sampling and fast leaf evaluators.

In a study by Zhang et al.~\cite{zhang2021subgame}, a similar approach used a perfect leaf evaluator at an extended depth compared to PIMC, focusing on 'common knowledge' to reduce costs without abstraction. 
Their research did not explore determination algorithms or key aspects central to our investigation, such as sub-game creation, budget considerations, and the influence of depth on strategy fusion and private/public observations. 
Additionally, their move ordering, prioritizing promising nodes over uniform state space exploration, may lead to strategy fusion and suboptimal decision-making, issues we address in our study.

\section{Conclusion and future works}
~\label{sec:conclusion}

In this paper, we introduce a novel online algorithm `Extended PIMC' for games with imperfect information. 
Building upon the foundation of PIMC, our approach postpones the perfect leaf evaluator to a deeper depth. 
Thanks to that, we have been able to successfully reduce past problems of PIMC and beat other online algorithms on multiple benchmarks. 
Especially, when benchmarks have hidden observation, significant performance improvements are observed. 
Furthermore, we conducted an in-depth analysis of various hyperparameters to provide a comprehensive understanding of their impact.

We enhance our research by presenting theoretical foundations for determinization-based algorithms that suffer from strategy fusion. 
We demonstrate that, in the worst case, increasing the depth does not increase the strategy fusion and in every case, there exists a depth $\depth$ such that the strategy fusion is strictly reduced.

As our algorithm is online, it has the advantage of being tested in a short period of time, especially in comparison to recent algorithms which need a domain-specific abstraction or a very high initialization cost due to neural networks. 
Even though, improving our algorithm by using deep learning could have led to superior performance. 
In particular, this could provide a better and faster approximation to the leaf or being able to remove the problem of non-locality by adding an inference system~\cite{rebstock_policy_2019}. 

In future research, it would be intriguing to assess our algorithm in conjunction with other determination algorithms. 
Particularly, exploring a trade-off between our approach, which ensures non-strategy fusion at a certain depth thanks to uniform sampling, and other algorithms that may explore the state space more efficiently but currently encounter issues with strategy fusion during exploration, could be beneficial. 
Such an investigation could shed light on the strengths and limitations of different approaches and potentially lead to the development of more robust and efficient algorithms for imperfect information games.

\bibliographystyle{IEEEtran}
\bibliography{Doctorat}

\newpage
\appendix

\section{Complementary experiments}
In the following, we expand upon the main paper's experiments by: (i) adding Card Game and Battleship; (ii) replicating the experiment that compares the depth  against IS-MCTS instead of PIMC. % and (iii) testing the impact of different exploration/exploitation strategies in EPIMC, acknowledging the risk of not guaranteeing reduced strategy fusion.

As in the main paper, the default settings for EPIMC include a fixed depth of 3, information set search for subgame resolution, and random rollouts as the perfect information leaf evaluator. The opponent uses one second of reasoning, while the tested algorithms use between 0.1 to 100 seconds.

As discussed in the main paper, Card Game and Battleship primarily involve public actions, and therefore as limited impact on our algorithms. 

\subsection{Postponing leaf evaluator}

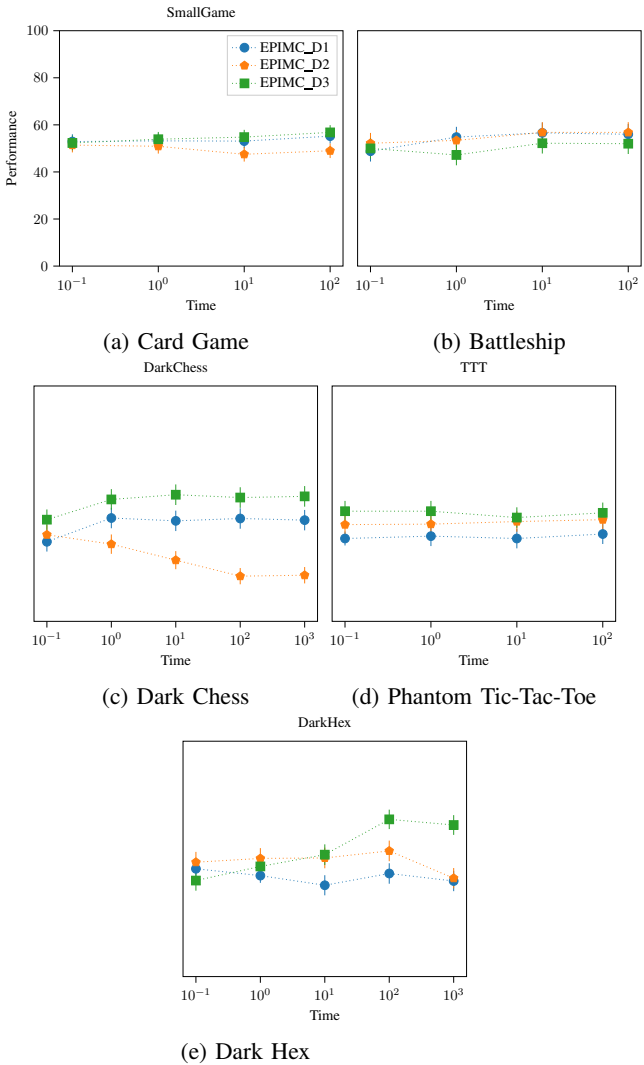
\begin{figure}[!htbp]
    \centering
    
        \subfloat[\centering Card Game]{\scalebox{0.55}{% This file was created with tikzplotlib v0.10.1.
\begin{tikzpicture}

\definecolor{darkgray176}{RGB}{176,176,176}
\definecolor{darkorange25512714}{RGB}{255,127,14}
\definecolor{forestgreen4416044}{RGB}{44,160,44}
\definecolor{lightgray204}{RGB}{204,204,204}
\definecolor{steelblue31119180}{RGB}{31,119,180}

\begin{axis}[
legend cell align={left},
legend style={fill opacity=0.8, draw opacity=1, text opacity=1, draw=lightgray204},
log basis x={10},
tick align=outside,
tick pos=left,
title={SmallGame},
x grid style={darkgray176},
xlabel={Time},
xmin=0.0707945784384138, xmax=141.253754462275,
xmode=log,
xtick style={color=black},
xtick={0.001,0.01,0.1,1,10,100,1000,10000},
xticklabels={
  $\displaystyle{10^{-3}}$,
  $\displaystyle{10^{-2}}$,
  $\displaystyle{10^{-1}}$,
  $\displaystyle{10^{0}}$,
  $\displaystyle{10^{1}}$,
  $\displaystyle{10^{2}}$,
  $\displaystyle{10^{3}}$,
  $\displaystyle{10^{4}}$
},
y grid style={darkgray176},
ylabel={Performance},
ymin=0, ymax=100,
ytick style={color=black}
]
\path [draw=steelblue31119180, semithick]
(axis cs:0.1,49.806184856201)
--(axis cs:0.1,55.993815143799);

\path [draw=steelblue31119180, semithick]
(axis cs:1,50.1073212232759)
--(axis cs:1,56.2926787767242);

\path [draw=steelblue31119180, semithick]
(axis cs:10,50.0069299678151)
--(axis cs:10,56.1930700321849);

\path [draw=steelblue31119180, semithick]
(axis cs:100,52.1177730232833)
--(axis cs:100,58.2822269767167);

\path [draw=darkorange25512714, semithick]
(axis cs:0.1,48.3021829518191)
--(axis cs:0.1,54.4978170481809);

\path [draw=darkorange25512714, semithick]
(axis cs:1,47.8014699769084)
--(axis cs:1,53.9985300230916);

\path [draw=darkorange25512714, semithick]
(axis cs:10,44.4048441073187)
--(axis cs:10,50.5951558926813);

\path [draw=darkorange25512714, semithick]
(axis cs:100,45.9015877614494)
--(axis cs:100,52.0984122385506);

\path [draw=forestgreen4416044, semithick]
(axis cs:0.1,49.2042484053142)
--(axis cs:0.1,55.3957515946859);

\path [draw=forestgreen4416044, semithick]
(axis cs:1,50.8104095313456)
--(axis cs:1,56.9895904686544);

\path [draw=forestgreen4416044, semithick]
(axis cs:10,51.7152812873781)
--(axis cs:10,57.884718712622);

\path [draw=forestgreen4416044, semithick]
(axis cs:100,53.7297615050293)
--(axis cs:100,59.8702384949707);

\addplot [semithick, steelblue31119180, dotted, mark=*, mark size=3, mark options={solid}]
table {%
0.1 52.9
1 53.2
10 53.1
100 55.2
};
\addlegendentry{EPIMC\_D1}
\addplot [semithick, darkorange25512714, dotted, mark=pentagon*, mark size=3, mark options={solid}]
table {%
0.1 51.4
1 50.9
10 47.5
100 49
};
\addlegendentry{EPIMC\_D2}
\addplot [semithick, forestgreen4416044, dotted, mark=square*, mark size=3, mark options={solid}]
table {%
0.1 52.3
1 53.9
10 54.8
100 56.8
};
\addlegendentry{EPIMC\_D3}
\end{axis}

\end{tikzpicture}}}
        \subfloat[\centering Battleship]{\scalebox{0.55}{% This file was created with tikzplotlib v0.10.1.
\begin{tikzpicture}

\definecolor{darkgray176}{RGB}{176,176,176}
\definecolor{darkorange25512714}{RGB}{255,127,14}
\definecolor{forestgreen4416044}{RGB}{44,160,44}
\definecolor{steelblue31119180}{RGB}{31,119,180}

\begin{axis}[
log basis x={10},
tick align=outside,
tick pos=left,
x grid style={darkgray176},
xlabel={Time},
xmin=0.0707945784384138, xmax=141.253754462275,
xmode=log,
xtick style={color=black},
xtick={0.001,0.01,0.1,1,10,100,1000,10000},
xticklabels={
  $\displaystyle{10^{-3}}$,
  $\displaystyle{10^{-2}}$,
  $\displaystyle{10^{-1}}$,
  $\displaystyle{10^{0}}$,
  $\displaystyle{10^{1}}$,
  $\displaystyle{10^{2}}$,
  $\displaystyle{10^{3}}$,
  $\displaystyle{10^{4}}$
},
ytick=\empty,
ymin=0, ymax=100
]
\path [draw=steelblue31119180, semithick]
(axis cs:0.1,44.4185691615638)
--(axis cs:0.1,53.1814308384362);

\path [draw=steelblue31119180, semithick]
(axis cs:1,50.437548960504)
--(axis cs:1,59.162451039496);

\path [draw=steelblue31119180, semithick]
(axis cs:10,52.2556565734279)
--(axis cs:10,60.9443434265721);

\path [draw=steelblue31119180, semithick]
(axis cs:100,51.6489765801596)
--(axis cs:100,60.3510234198404);

\path [draw=darkorange25512714, semithick]
(axis cs:0.1,47.8215512664872)
--(axis cs:0.1,56.5784487335128);

\path [draw=darkorange25512714, semithick]
(axis cs:1,49.0274512915234)
--(axis cs:1,57.7725487084766);

\path [draw=darkorange25512714, semithick]
(axis cs:10,52.4580270806925)
--(axis cs:10,61.1419729193075);

\path [draw=darkorange25512714, semithick]
(axis cs:100,52.4580270806925)
--(axis cs:100,61.1419729193075);

\path [draw=forestgreen4416044, semithick]
(axis cs:0.1,45.6173067641004)
--(axis cs:0.1,54.3826932358996);

\path [draw=forestgreen4416044, semithick]
(axis cs:1,42.8241842232562)
--(axis cs:1,51.5758157767438);

\path [draw=forestgreen4416044, semithick]
(axis cs:10,47.8215512664872)
--(axis cs:10,56.5784487335128);

\path [draw=forestgreen4416044, semithick]
(axis cs:100,47.6208143222741)
--(axis cs:100,56.3791856777259);

\addplot [semithick, steelblue31119180, dotted, mark=*, mark size=3, mark options={solid}]
table {%
0.1 48.8
1 54.8
10 56.6
100 56
};
% \addlegendentry{DL_D3}
\addplot [semithick, darkorange25512714, dotted, mark=pentagon*, mark size=3, mark options={solid}]
table {%
0.1 52.2
1 53.4
10 56.8
100 56.8
};
% \addlegendentry{DL_D3}
\addplot [semithick, forestgreen4416044, dotted, mark=square*, mark size=3, mark options={solid}]
table {%
0.1 50
1 47.2
10 52.2
100 52
};
% \addlegendentry{DL_D3}
\end{axis}

\end{tikzpicture}}}
    
        \subfloat[\centering Dark Chess]{\scalebox{0.55}{% This file was created with tikzplotlib v0.10.1.
\begin{tikzpicture}

\definecolor{darkgray176}{RGB}{176,176,176}
\definecolor{darkorange25512714}{RGB}{255,127,14}
\definecolor{forestgreen4416044}{RGB}{44,160,44}
\definecolor{steelblue31119180}{RGB}{31,119,180}

\begin{axis}[
log basis x={10},
tick align=outside,
tick pos=left,
title={DarkChess},
x grid style={darkgray176},
xlabel={Time},
xmin=0.0630957344480193, xmax=1584.89319246111,
xmode=log,
xtick style={color=black},
xtick={0.001,0.01,0.1,1,10,100,1000,10000,100000},
xticklabels={
  $\displaystyle{10^{-3}}$,
  $\displaystyle{10^{-2}}$,
  $\displaystyle{10^{-1}}$,
  $\displaystyle{10^{0}}$,
  $\displaystyle{10^{1}}$,
  $\displaystyle{10^{2}}$,
  $\displaystyle{10^{3}}$,
  $\displaystyle{10^{4}}$,
  $\displaystyle{10^{5}}$
},
ytick=\empty,
ymin=0, ymax=100
]
\path [draw=steelblue31119180, semithick]
(axis cs:0.1,29.6537220315083)
--(axis cs:0.1,37.9462779684917);

\path [draw=steelblue31119180, semithick]
(axis cs:1,39.5500450429918)
--(axis cs:1,48.2499549570082);

\path [draw=steelblue31119180, semithick]
(axis cs:10,38.3642691190527)
--(axis cs:10,47.0357308809473);

\path [draw=steelblue31119180, semithick]
(axis cs:100,39.3522357708818)
--(axis cs:100,48.0477642291182);

\path [draw=steelblue31119180, semithick]
(axis cs:1000,38.6604697028365)
--(axis cs:1000,47.3395302971635);

\path [draw=darkorange25512714, semithick]
(axis cs:0.1,32.5727929750248)
--(axis cs:0.1,41.0272070249752);

\path [draw=darkorange25512714, semithick]
(axis cs:1,28.6847840746809)
--(axis cs:1,36.9152159253191);

\path [draw=darkorange25512714, semithick]
(axis cs:10,22.155201331669)
--(axis cs:10,29.844798668331);

\path [draw=darkorange25512714, semithick]
(axis cs:100,15.7475502679981)
--(axis cs:100,22.6524497320019);

\path [draw=darkorange25512714, semithick]
(axis cs:1000,16.12041756988)
--(axis cs:1000,23.07958243012);

\path [draw=forestgreen4416044, semithick]
(axis cs:0.1,38.8580270806925)
--(axis cs:0.1,47.5419729193075);

\path [draw=forestgreen4416044, semithick]
(axis cs:1,47.4201476700692)
--(axis cs:1,56.1798523299308);

\path [draw=forestgreen4416044, semithick]
(axis cs:10,49.4299823121639)
--(axis cs:10,58.1700176878361);

\path [draw=forestgreen4416044, semithick]
(axis cs:100,48.2232361763513)
--(axis cs:100,56.9767638236487);

\path [draw=forestgreen4416044, semithick]
(axis cs:1000,48.7257384111144)
--(axis cs:1000,57.4742615888856);

\addplot [semithick, steelblue31119180, dotted, mark=*, mark size=3, mark options={solid}]
table {%
0.1 33.8
1 43.9
10 42.7
100 43.7
1000 43
};
% \addlegendentry{DL_D3}
\addplot [semithick, darkorange25512714, dotted, mark=pentagon*, mark size=3, mark options={solid}]
table {%
0.1 36.8
1 32.8
10 26
100 19.2
1000 19.6
};
% \addlegendentry{DL_D3}
\addplot [semithick, forestgreen4416044, dotted, mark=square*, mark size=3, mark options={solid}]
table {%
0.1 43.2
1 51.8
10 53.8
100 52.6
1000 53.1
};
% \addlegendentry{DL_D3}
\end{axis}

\end{tikzpicture}}}
        \subfloat[\centering Phantom Tic-Tac-Toe]{\scalebox{0.55}{% This file was created with tikzplotlib v0.10.1.
\begin{tikzpicture}

\definecolor{darkgray176}{RGB}{176,176,176}
\definecolor{darkorange25512714}{RGB}{255,127,14}
\definecolor{forestgreen4416044}{RGB}{44,160,44}
\definecolor{steelblue31119180}{RGB}{31,119,180}

\begin{axis}[
log basis x={10},
tick align=outside,
tick pos=left,
title={TTT},
x grid style={darkgray176},
xlabel={Time},
xmin=0.0707945784384138, xmax=141.253754462275,
xmode=log,
xtick style={color=black},
xtick={0.001,0.01,0.1,1,10,100,1000,10000},
xticklabels={
  $\displaystyle{10^{-3}}$,
  $\displaystyle{10^{-2}}$,
  $\displaystyle{10^{-1}}$,
  $\displaystyle{10^{0}}$,
  $\displaystyle{10^{1}}$,
  $\displaystyle{10^{2}}$,
  $\displaystyle{10^{3}}$,
  $\displaystyle{10^{4}}$
},
ytick=\empty,
ymin=0, ymax=100
]
\path [draw=steelblue31119180, semithick]
(axis cs:0.1,32.2398419069246)
--(axis cs:0.1,38.1601580930754);

\path [draw=steelblue31119180, semithick]
(axis cs:1,31.9875409328992)
--(axis cs:1,40.4124590671008);

\path [draw=steelblue31119180, semithick]
(axis cs:10,31.0137042780042)
--(axis cs:10,39.3862957219958);

\path [draw=steelblue31119180, semithick]
(axis cs:100,32.8656832088281)
--(axis cs:100,41.3343167911719);

\path [draw=darkorange25512714, semithick]
(axis cs:0.1,36.7872962392485)
--(axis cs:0.1,45.4127037607515);

\path [draw=darkorange25512714, semithick]
(axis cs:1,36.9841618899685)
--(axis cs:1,45.6158381100315);

\path [draw=darkorange25512714, semithick]
(axis cs:10,38.0682314964901)
--(axis cs:10,46.7317685035099);

\path [draw=darkorange25512714, semithick]
(axis cs:100,38.8580270806925)
--(axis cs:100,47.5419729193075);

\path [draw=forestgreen4416044, semithick]
(axis cs:0.1,42.4262917298933)
--(axis cs:0.1,51.1737082701067);

\path [draw=forestgreen4416044, semithick]
(axis cs:1,42.4262917298933)
--(axis cs:1,51.1737082701067);

\path [draw=forestgreen4416044, semithick]
(axis cs:10,39.7479260337168)
--(axis cs:10,48.4520739662832);

\path [draw=forestgreen4416044, semithick]
(axis cs:100,41.7306592570503)
--(axis cs:100,50.4693407429497);

\addplot [semithick, steelblue31119180, dotted, mark=*, mark size=3, mark options={solid}]
table {%
0.1 35.2
1 36.2
10 35.2
100 37.1
};
% \addlegendentry{DL_D3}
\addplot [semithick, darkorange25512714, dotted, mark=pentagon*, mark size=3, mark options={solid}]
table {%
0.1 41.1
1 41.3
10 42.4
100 43.2
};
% \addlegendentry{DL_D3}
\addplot [semithick, forestgreen4416044, dotted, mark=square*, mark size=3, mark options={solid}]
table {%
0.1 46.8
1 46.8
10 44.1
100 46.1
};
% \addlegendentry{DL_D3}
\end{axis}

\end{tikzpicture}}}

        \subfloat[\centering Dark Hex \hspace{20mm}~]{\scalebox{0.55}{% This file was created with tikzplotlib v0.10.1.
\begin{tikzpicture}

\definecolor{darkgray176}{RGB}{176,176,176}
\definecolor{darkorange25512714}{RGB}{255,127,14}
\definecolor{forestgreen4416044}{RGB}{44,160,44}
\definecolor{steelblue31119180}{RGB}{31,119,180}

\begin{axis}[
log basis x={10},
tick align=outside,
tick pos=left,
title={DarkHex},
x grid style={darkgray176},
xlabel={Time},
xmin=0.0630957344480193, xmax=1584.89319246111,
xmode=log,
xtick style={color=black},
xtick={0.001,0.01,0.1,1,10,100,1000,10000,100000},
xticklabels={
  $\displaystyle{10^{-3}}$,
  $\displaystyle{10^{-2}}$,
  $\displaystyle{10^{-1}}$,
  $\displaystyle{10^{0}}$,
  $\displaystyle{10^{1}}$,
  $\displaystyle{10^{2}}$,
  $\displaystyle{10^{3}}$,
  $\displaystyle{10^{4}}$,
  $\displaystyle{10^{5}}$
},
ytick=\empty,
ymin=0, ymax=100
]
\path [draw=steelblue31119180, semithick]
(axis cs:0.1,41.4327962777081)
--(axis cs:0.1,50.1672037222919);

\path [draw=steelblue31119180, semithick]
(axis cs:1,39.8323714462145)
--(axis cs:1,45.9676285537855);

\path [draw=steelblue31119180, semithick]
(axis cs:10,34.5286747499166)
--(axis cs:10,43.0713252500834);

\path [draw=steelblue31119180, semithick]
(axis cs:100,39.4511314354191)
--(axis cs:100,48.1488685645809);

\path [draw=steelblue31119180, semithick]
(axis cs:1000,36.2954544434981)
--(axis cs:1000,44.9045455565019);

\path [draw=darkorange25512714, semithick]
(axis cs:0.1,44.219025116712)
--(axis cs:0.1,52.980974883288);

\path [draw=darkorange25512714, semithick]
(axis cs:1,45.8173418257865)
--(axis cs:1,54.5826581742135);

\path [draw=darkorange25512714, semithick]
(axis cs:10,46.017447012528)
--(axis cs:10,54.782552987472);

\path [draw=darkorange25512714, semithick]
(axis cs:100,49.0274512915234)
--(axis cs:100,57.7725487084766);

\path [draw=darkorange25512714, semithick]
(axis cs:1000,37.4766469457145)
--(axis cs:1000,46.1233530542855);

\path [draw=forestgreen4416044, semithick]
(axis cs:0.1,36.4921358015833)
--(axis cs:0.1,45.1078641984167);

\path [draw=forestgreen4416044, semithick]
(axis cs:1,42.4262917298933)
--(axis cs:1,51.1737082701067);

\path [draw=forestgreen4416044, semithick]
(axis cs:10,47.4201476700692)
--(axis cs:10,56.1798523299308);

\path [draw=forestgreen4416044, semithick]
(axis cs:100,62.67210784637)
--(axis cs:100,70.92789215363);

\path [draw=forestgreen4416044, semithick]
(axis cs:1000,60.2029996845366)
--(axis cs:1000,68.5970003154634);

\addplot [semithick, steelblue31119180, dotted, mark=*, mark size=3, mark options={solid}]
table {%
0.1 45.8
1 42.9
10 38.8
100 43.8
1000 40.6
};
% \addlegendentry{DL_D3}
\addplot [semithick, darkorange25512714, dotted, mark=pentagon*, mark size=3, mark options={solid}]
table {%
0.1 48.6
1 50.2
10 50.4
100 53.4
1000 41.8
};
% \addlegendentry{DL_D3}
\addplot [semithick, forestgreen4416044, dotted, mark=square*, mark size=3, mark options={solid}]
table {%
0.1 40.8
1 46.8
10 51.8
100 66.8
1000 64.4
};
% \addlegendentry{DL_D3}
\end{axis}

\end{tikzpicture}}}
    
        \caption{\centering Winning rate of online algorithms. The opponent is IS-MCTS with one second of budget.}
    
    \label{fig:vsISfull}
\end{figure}

As discussed in the main paper, Card Game and Battleship primarily involve public actions, resulting in similar performance across different methods. Conversely, Dark Chess, Dark Hex, and Phantom Tic-Tac-Toe involve private actions, where changing the depth significantly impacts performance. Increasing the depth to 3 notably improves performance in these three games.

Interestingly, in Dark Chess, we observe that performance at depth 2 is lower than that of PIMC. This can be attributed to the fact that while increasing depth reduces strategy fusion, the fused strategy at depth 1 may have been advantageous against IS-MCTS. Thus, moving to depth 2 removed this beneficial fusion, leading to a reduction in performance.

\subsection{Subgame Resolution}
\label{exp:subgameappendix}

\begin{figure}[!htbp]
    \centering
    
        \subfloat[\centering Card Game]{\scalebox{0.55}{% This file was created with tikzplotlib v0.10.1.
\begin{tikzpicture}

\definecolor{darkgray176}{RGB}{176,176,176}
\definecolor{darkorange25512714}{RGB}{255,127,14}
\definecolor{lightgray204}{RGB}{204,204,204}
\definecolor{steelblue31119180}{RGB}{31,119,180}

\begin{axis}[
legend cell align={left},
legend style={fill opacity=0.8, draw opacity=1, text opacity=1, draw=lightgray204},
log basis x={10},
tick align=outside,
tick pos=left,
x grid style={darkgray176},
xlabel={Budget},
xmin=0.0707945784384138, xmax=141.253754462275,
xmode=log,
xtick style={color=black},
xtick={0.001,0.01,0.1,1,10,100,1000,10000},
xticklabels={
  \(\displaystyle {10^{-3}}\),
  \(\displaystyle {10^{-2}}\),
  \(\displaystyle {10^{-1}}\),
  \(\displaystyle {10^{0}}\),
  \(\displaystyle {10^{1}}\),
  \(\displaystyle {10^{2}}\),
  \(\displaystyle {10^{3}}\),
  \(\displaystyle {10^{4}}\)
},
y grid style={darkgray176},
ylabel={Winning Rate},
ymin=0, ymax=100,
ytick style={color=black}
]
\path [draw=steelblue31119180, semithick]
(axis cs:0.1,38.9408926792281)
--(axis cs:0.1,45.0591073207719);

\path [draw=steelblue31119180, semithick]
(axis cs:1,38.6439646310947)
--(axis cs:1,44.7560353689053);

\path [draw=steelblue31119180, semithick]
(axis cs:10,45.1029767169102)
--(axis cs:10,51.2970232830897);

\path [draw=steelblue31119180, semithick]
(axis cs:100,46.8009740911054)
--(axis cs:100,52.9990259088946);

\path [draw=darkorange25512714, semithick]
(axis cs:0.1,45.8017179495727)
--(axis cs:0.1,51.9982820504273);

\path [draw=darkorange25512714, semithick]
(axis cs:1,47.3010670636492)
--(axis cs:1,53.4989329363508);

\path [draw=darkorange25512714, semithick]
(axis cs:10,47.50119103138)
--(axis cs:10,53.69880896862);

\path [draw=darkorange25512714, semithick]
(axis cs:100,48.6027596515608)
--(axis cs:100,54.7972403484392);

\addplot [semithick, steelblue31119180, dotted, mark=*, mark size=3, mark options={solid}]
table {%
0.1 42
1 41.7
10 48.2
100 49.9
};
\addlegendentry{EPIMC\_CFR}
\addplot [semithick, darkorange25512714, dotted, mark=square*, mark size=3, mark options={solid}]
table {%
0.1 48.9
1 50.4
10 50.6
100 51.7
};
\addlegendentry{EPIMC\_ISS}
\end{axis}

\end{tikzpicture}}}
        \subfloat[\centering Battleship]{\scalebox{0.55}{% This file was created with tikzplotlib v0.10.1.
\begin{tikzpicture}

\definecolor{darkgray176}{RGB}{176,176,176}
\definecolor{darkorange25512714}{RGB}{255,127,14}
\definecolor{lightgray204}{RGB}{204,204,204}
\definecolor{steelblue31119180}{RGB}{31,119,180}

\begin{axis}[
legend cell align={left},
legend style={fill opacity=0.8, draw opacity=1, text opacity=1, draw=lightgray204},
log basis x={10},
tick align=outside,
tick pos=left,
x grid style={darkgray176},
xlabel={Budget},
xmin=0.0707945784384138, xmax=141.253754462275,
xmode=log,
xtick style={color=black},
xtick={0.001,0.01,0.1,1,10,100,1000,10000},
xticklabels={
  \(\displaystyle {10^{-3}}\),
  \(\displaystyle {10^{-2}}\),
  \(\displaystyle {10^{-1}}\),
  \(\displaystyle {10^{0}}\),
  \(\displaystyle {10^{1}}\),
  \(\displaystyle {10^{2}}\),
  \(\displaystyle {10^{3}}\),
  \(\displaystyle {10^{4}}\)
},
y grid style={darkgray176},
ymin=0, ymax=100,
ytick=\empty,
ytick style={color=black}
]
\path [draw=steelblue31119180, semithick]
(axis cs:0.1,40.4272662985543)
--(axis cs:0.1,46.5727337014457);

\path [draw=steelblue31119180, semithick]
(axis cs:1,44.8624707457956)
--(axis cs:1,49.2375292542044);

\path [draw=steelblue31119180, semithick]
(axis cs:10,48.4023627714014)
--(axis cs:10,54.5976372285986);

\path [draw=steelblue31119180, semithick]
(axis cs:100,48.903448111205)
--(axis cs:100,55.096551888795);

\path [draw=darkorange25512714, semithick]
(axis cs:0.1,43.5081411571677)
--(axis cs:0.1,49.6918588428323);

\path [draw=darkorange25512714, semithick]
(axis cs:1,44.903448111205)
--(axis cs:1,51.096551888795);

\path [draw=darkorange25512714, semithick]
(axis cs:10,45.0032062025379)
--(axis cs:10,51.1967937974621);

\path [draw=darkorange25512714, semithick]
(axis cs:100,44.2054895805637)
--(axis cs:100,50.3945104194363);

\addplot [semithick, steelblue31119180, dotted, mark=*, mark size=3, mark options={solid}]
table {%
0.1 43.5
1 47.05
10 51.5
100 52
};
%\addlegendentry{DL_ISS}
\addplot [semithick, darkorange25512714, dotted, mark=square*, mark size=3, mark options={solid}]
table {%
0.1 46.6
1 48
10 48.1
100 47.3
};
%\addlegendentry{DL_ISS}
\end{axis}

\end{tikzpicture}}}
     
        % \subfloat[\centering Dark Chess]{\scalebox{0.6}{\input{Image/DarkChess_CFR}}}
        % \subfloat[\centering Phantom Tic-Tac-Toe]{\scalebox{0.6}{\input{Image/TTT_CFR}}}
        %  \subfloat[\centering Dark Hex \hspace{20mm}~]{\scalebox{0.6}{\input{Image/DarkHex_CFR}}}
    
        \caption{\centering Winning rate of EPIMC when the subgame is `CFR' or `ISS'. The opponent is PIMC with one second of budget.}
    
    \label{fig:_subgameappendix}
\end{figure}
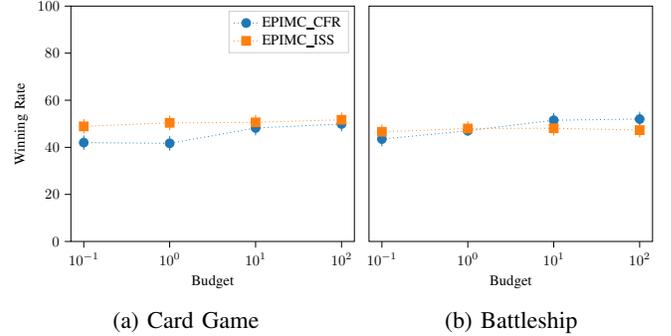

In Figure \ref{fig:_subgameappendix}, we analyze the impact of using CFR+ instead of Information Set Search (ISS) in the subgame resolution. CFR+ is used with $1000$ iterations. Compared with the main paper, Card Game and Battleship have been added. As the game have public actions, the performance of the two methods is similar.

\subsection{Leaf Evaluator}

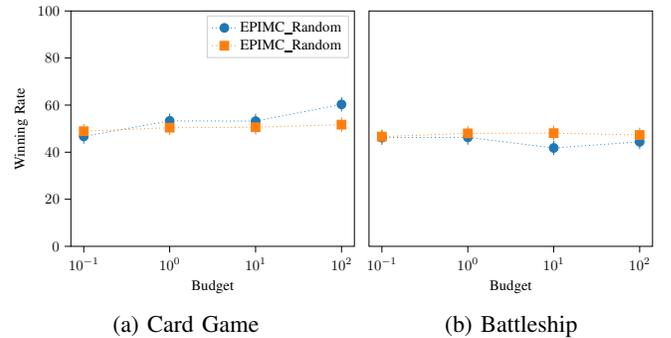
\begin{figure}[!htbp]
\centering

    \subfloat[\centering Card Game]{\scalebox{0.55}{% This file was created with tikzplotlib v0.10.1.
\begin{tikzpicture}

\definecolor{darkgray176}{RGB}{176,176,176}
\definecolor{darkorange25512714}{RGB}{255,127,14}
\definecolor{lightgray204}{RGB}{204,204,204}
\definecolor{steelblue31119180}{RGB}{31,119,180}

\begin{axis}[
legend cell align={left},
legend style={fill opacity=0.8, draw opacity=1, text opacity=1, draw=lightgray204},
log basis x={10},
tick align=outside,
tick pos=left,
x grid style={darkgray176},
xlabel={Budget},
xmin=0.0707945784384138, xmax=141.253754462275,
xmode=log,
xtick style={color=black},
xtick={0.001,0.01,0.1,1,10,100,1000,10000},
xticklabels={
  \(\displaystyle {10^{-3}}\),
  \(\displaystyle {10^{-2}}\),
  \(\displaystyle {10^{-1}}\),
  \(\displaystyle {10^{0}}\),
  \(\displaystyle {10^{1}}\),
  \(\displaystyle {10^{2}}\),
  \(\displaystyle {10^{3}}\),
  \(\displaystyle {10^{4}}\)
},
y grid style={darkgray176},
ylabel={Winning Rate},
ymin=0, ymax=100,
ytick style={color=black}
]
\path [draw=steelblue31119180, semithick]
(axis cs:0.1,43.6077249514314)
--(axis cs:0.1,49.7922750485686);

\path [draw=steelblue31119180, semithick]
(axis cs:1,50.2077249514314)
--(axis cs:1,56.3922750485686);

\path [draw=steelblue31119180, semithick]
(axis cs:10,50.1073212232759)
--(axis cs:10,56.2926787767242);

\path [draw=steelblue31119180, semithick]
(axis cs:100,57.2674359601156)
--(axis cs:100,63.3325640398844);

\path [draw=darkorange25512714, semithick]
(axis cs:0.1,45.8017179495727)
--(axis cs:0.1,51.9982820504273);

\path [draw=darkorange25512714, semithick]
(axis cs:1,47.3010670636492)
--(axis cs:1,53.4989329363508);

\path [draw=darkorange25512714, semithick]
(axis cs:10,47.50119103138)
--(axis cs:10,53.69880896862);

\path [draw=darkorange25512714, semithick]
(axis cs:100,48.6027596515608)
--(axis cs:100,54.7972403484392);

\addplot [semithick, steelblue31119180, dotted, mark=*, mark size=3, mark options={solid}]
table {%
0.1 46.7
1 53.3
10 53.2
100 60.3
};
\addlegendentry{EPIMC\_Random}
\addplot [semithick, darkorange25512714, dotted, mark=square*, mark size=3, mark options={solid}]
table {%
0.1 48.9
1 50.4
10 50.6
100 51.7
};
\addlegendentry{EPIMC\_Random}
\end{axis}

\end{tikzpicture}}}
    \subfloat[\centering Battleship]{\scalebox{0.55}{% This file was created with tikzplotlib v0.10.1.
\begin{tikzpicture}

\definecolor{darkgray176}{RGB}{176,176,176}
\definecolor{darkorange25512714}{RGB}{255,127,14}
\definecolor{lightgray204}{RGB}{204,204,204}
\definecolor{steelblue31119180}{RGB}{31,119,180}

\begin{axis}[
legend cell align={left},
legend style={fill opacity=0.8, draw opacity=1, text opacity=1, draw=lightgray204},
log basis x={10},
tick align=outside,
tick pos=left,
x grid style={darkgray176},
xlabel={Budget},
xmin=0.0707945784384138, xmax=141.253754462275,
xmode=log,
xtick style={color=black},
xtick={0.001,0.01,0.1,1,10,100,1000,10000},
xticklabels={
  \(\displaystyle {10^{-3}}\),
  \(\displaystyle {10^{-2}}\),
  \(\displaystyle {10^{-1}}\),
  \(\displaystyle {10^{0}}\),
  \(\displaystyle {10^{1}}\),
  \(\displaystyle {10^{2}}\),
  \(\displaystyle {10^{3}}\),
  \(\displaystyle {10^{4}}\)
},
y grid style={darkgray176},
ymin=0, ymax=100,
ytick=\empty,
ytick style={color=black}
]
\path [draw=steelblue31119180, semithick]
(axis cs:0.1,43.109930859026)
--(axis cs:0.1,49.290069140974);

\path [draw=steelblue31119180, semithick]
(axis cs:1,43.2094646910284)
--(axis cs:1,49.3905353089716);

\path [draw=steelblue31119180, semithick]
(axis cs:10,38.7429277378511)
--(axis cs:10,44.8570722621489);

\path [draw=steelblue31119180, semithick]
(axis cs:100,41.3204663625802)
--(axis cs:100,47.4795336374198);

\path [draw=darkorange25512714, semithick]
(axis cs:0.1,43.5081411571677)
--(axis cs:0.1,49.6918588428323);

\path [draw=darkorange25512714, semithick]
(axis cs:1,44.903448111205)
--(axis cs:1,51.096551888795);

\path [draw=darkorange25512714, semithick]
(axis cs:10,45.0032062025379)
--(axis cs:10,51.1967937974621);

\path [draw=darkorange25512714, semithick]
(axis cs:100,44.2054895805637)
--(axis cs:100,50.3945104194363);

\addplot [semithick, steelblue31119180, dotted, mark=*, mark size=3, mark options={solid}]
table {%
0.1 46.2
1 46.3
10 41.8
100 44.4
};
%\addlegendentry{DL_Random}
\addplot [semithick, darkorange25512714, dotted, mark=square*, mark size=3, mark options={solid}]
table {%
0.1 46.6
1 48
10 48.1
100 47.3
};
%\addlegendentry{DL_Random}
\end{axis}

\end{tikzpicture}}}
    
    % \subfloat[\centering Phantom Tic-Tac-Toe]{\scalebox{0.55}{\input{Image/TTT_MaxN}}}
    % \subfloat[\centering Dark Hex \hspace{20mm}~]{\scalebox{0.55}{\input{Image/DarkHex_MaxN}}}

    \caption{\centering Winning rate of EPIMC when the perfect information leaf evaluator is `Minimax' or `RandomRollout'. The opponent is PIMC with one second of budget.}

\label{fig:_leafappendix}
\end{figure}

In Figure \ref{fig:_leafappendix}, we analyze the impact of using Minimax with Alpha-Beta instead of using Random Rollout for the perfect information leaf evaluator. Compared with the main paper, Card Game and Battleship have been added. As the game have public actions, the performance of the two methods is similar.

\subsection{Against other online algorithms}

In Figure \ref{fig:vsPIMCfull}, we add the game Battleship and Card Game against other online algorithms. As in the main paper, we observe that EPIMC and IS-MCTS obtain the best performance against other online algorithms and, OOS and RecPIMC obtain poorer performance. However, RecPIMC achieves performance close to EPIMC and IS-MCTS in Battleship. 

\begin{figure}[!htbp]
\centering
	\subfloat[\centering Card Game]{\scalebox{0.55}{% This file was created with tikzplotlib v0.10.1.
\begin{tikzpicture}

\definecolor{crimson2143940}{RGB}{214,39,40}
\definecolor{darkgray176}{RGB}{176,176,176}
\definecolor{darkorange25512714}{RGB}{255,127,14}
\definecolor{forestgreen4416044}{RGB}{44,160,44}
\definecolor{lightgray204}{RGB}{204,204,204}
\definecolor{mediumpurple148103189}{RGB}{148,103,189}
\definecolor{sienna1408675}{RGB}{140,86,75}
\definecolor{steelblue31119180}{RGB}{31,119,180}

\begin{axis}[
legend cell align={left},
legend style={fill opacity=0.8, draw opacity=1, text opacity=1, draw=lightgray204},
log basis x={10},
tick align=outside,
tick pos=left,
x grid style={darkgray176},
xlabel={Budget},
xmin=0.0707945784384138, xmax=141.253754462275,
xmode=log,
xtick style={color=black},
xtick={0.001,0.01,0.1,1,10,100,1000,10000},
xticklabels={
  \(\displaystyle {10^{-3}}\),
  \(\displaystyle {10^{-2}}\),
  \(\displaystyle {10^{-1}}\),
  \(\displaystyle {10^{0}}\),
  \(\displaystyle {10^{1}}\),
  \(\displaystyle {10^{2}}\),
  \(\displaystyle {10^{3}}\),
  \(\displaystyle {10^{4}}\)
},
y grid style={darkgray176},
ylabel={Winning Rate},
ymin=0, ymax=100,
ytick style={color=black}
]
\path [draw=steelblue31119180, semithick]
(axis cs:0.1,46.5010670636492)
--(axis cs:0.1,52.6989329363508);

\path [draw=steelblue31119180, semithick]
(axis cs:1,46.700992685391)
--(axis cs:1,52.899007314609);

\path [draw=steelblue31119180, semithick]
(axis cs:10,47.0009740911054)
--(axis cs:10,53.1990259088946);

\path [draw=steelblue31119180, semithick]
(axis cs:100,47.0009740911054)
--(axis cs:100,53.1990259088946);

\path [draw=darkorange25512714, semithick]
(axis cs:0.1,45.4023627714014)
--(axis cs:0.1,51.5976372285986);

\path [draw=darkorange25512714, semithick]
(axis cs:1,46.5010670636492)
--(axis cs:1,52.6989329363508);

\path [draw=darkorange25512714, semithick]
(axis cs:10,43.2094646910284)
--(axis cs:10,49.3905353089716);

\path [draw=darkorange25512714, semithick]
(axis cs:100,43.2094646910284)
--(axis cs:100,49.3905353089716);

\path [draw=forestgreen4416044, semithick]
(axis cs:0.1,45.8017179495727)
--(axis cs:0.1,51.9982820504273);

\path [draw=forestgreen4416044, semithick]
(axis cs:1,47.3010670636492)
--(axis cs:1,53.4989329363508);

\path [draw=forestgreen4416044, semithick]
(axis cs:10,47.50119103138)
--(axis cs:10,53.69880896862);

\path [draw=forestgreen4416044, semithick]
(axis cs:100,48.6027596515608)
--(axis cs:100,54.7972403484392);

\path [draw=crimson2143940, semithick]
(axis cs:0.1,35.7797169510127)
--(axis cs:0.1,41.8202830489873);

\path [draw=crimson2143940, semithick]
(axis cs:1,43.7073212232759)
--(axis cs:1,49.8926787767242);

\path [draw=crimson2143940, semithick]
(axis cs:10,47.4011228485143)
--(axis cs:10,53.5988771514857);

\path [draw=crimson2143940, semithick]
(axis cs:100,48.3021829518191)
--(axis cs:100,54.4978170481809);

\path [draw=mediumpurple148103189, semithick]
(axis cs:0.1,15.2396521476698)
--(axis cs:0.1,19.9603478523302);

\path [draw=mediumpurple148103189, semithick]
(axis cs:1,15.3343928610186)
--(axis cs:1,20.0656071389814);

\path [draw=mediumpurple148103189, semithick]
(axis cs:10,15.4291614445517)
--(axis cs:10,20.1708385554483);

\path [draw=mediumpurple148103189, semithick]
(axis cs:100,14.577263548312)
--(axis cs:100,19.222736451688);

\path [draw=sienna1408675, semithick]
(axis cs:0.1,15.8085108371561)
--(axis cs:0.1,20.5914891628439);

\path [draw=sienna1408675, semithick]
(axis cs:1,15.9983480885024)
--(axis cs:1,20.8016519114976);

\path [draw=sienna1408675, semithick]
(axis cs:10,15.523957714181)
--(axis cs:10,20.276042285819);

\path [draw=sienna1408675, semithick]
(axis cs:100,16.093306749916)
--(axis cs:100,20.906693250084);

\addplot [semithick, steelblue31119180, dotted, mark=*, mark size=3, mark options={solid}]
table {%
0.1 49.6
1 49.8
10 50.1
100 50.1
};
\addlegendentry{PIMC}
\addplot [semithick, darkorange25512714, dotted, mark=pentagon*, mark size=3, mark options={solid}]
table {%
0.1 48.5
1 49.6
10 46.3
100 46.3
};
\addlegendentry{EPIMC\_D2}
\addplot [semithick, forestgreen4416044, dotted, mark=square*, mark size=3, mark options={solid}]
table {%
0.1 48.9
1 50.4
10 50.6
100 51.7
};
\addlegendentry{EPIMC\_D3}
\addplot [semithick, sienna1408675, dotted, mark=triangle*, mark size=3, mark options={solid}]
table {%
0.1 38.8
1 46.8
10 50.5
100 51.4
};
\addlegendentry{IS-MCTS}
\addplot [semithick, crimson2143940, dotted, mark=diamond*, mark size=3, mark options={solid}]
table {%
0.1 17.6
1 17.7
10 17.8
100 16.9
};
\addlegendentry{OOS2}
\addplot [semithick, mediumpurple148103189, dotted, mark=oplus, mark size=3, mark options={solid}]
table {%
0.1 18.2
1 18.4
10 17.9
100 18.5
};
\addlegendentry{IIMC}
\addplot [semithick, black, dotted, mark=halfcircle, mark size=3, mark options={solid}]
table {%
0.1 16.9
1 16.9
10 16.9
100 16.9
};
\addlegendentry{Random}

\end{axis}

\end{tikzpicture}}}
	\subfloat[\centering Battleship]{\scalebox{0.55}{% This file was created with tikzplotlib v0.10.1.
\begin{tikzpicture}

\definecolor{crimson2143940}{RGB}{214,39,40}
\definecolor{darkgray176}{RGB}{176,176,176}
\definecolor{darkorange25512714}{RGB}{255,127,14}
\definecolor{forestgreen4416044}{RGB}{44,160,44}
\definecolor{lightgray204}{RGB}{204,204,204}
\definecolor{mediumpurple148103189}{RGB}{148,103,189}
\definecolor{sienna1408675}{RGB}{140,86,75}
\definecolor{steelblue31119180}{RGB}{31,119,180}

\begin{axis}[
legend cell align={left},
legend style={fill opacity=0.8, draw opacity=1, text opacity=1, draw=lightgray204},
log basis x={10},
tick align=outside,
tick pos=left,
x grid style={darkgray176},
xlabel={Budget},
xmin=0.0707945784384138, xmax=141.253754462275,
xmode=log,
xtick style={color=black},
xtick={0.001,0.01,0.1,1,10,100,1000,10000},
xticklabels={
  \(\displaystyle {10^{-3}}\),
  \(\displaystyle {10^{-2}}\),
  \(\displaystyle {10^{-1}}\),
  \(\displaystyle {10^{0}}\),
  \(\displaystyle {10^{1}}\),
  \(\displaystyle {10^{2}}\),
  \(\displaystyle {10^{3}}\),
  \(\displaystyle {10^{4}}\)
},
y grid style={darkgray176},
ymin=0, ymax=100,
ytick=\empty,
ytick style={color=black}
]
\path [draw=steelblue31119180, semithick]
(axis cs:0.1,46.601023676115)
--(axis cs:0.1,52.798976323885);

\path [draw=steelblue31119180, semithick]
(axis cs:1,46.2012716130645)
--(axis cs:1,52.3987283869355);

\path [draw=steelblue31119180, semithick]
(axis cs:10,44.006184856201)
--(axis cs:10,50.193815143799);

\path [draw=steelblue31119180, semithick]
(axis cs:100,45.6020155429699)
--(axis cs:100,51.7979844570301);

\path [draw=darkorange25512714, semithick]
(axis cs:0.1,44.5045400367635)
--(axis cs:0.1,50.6954599632365);

\path [draw=darkorange25512714, semithick]
(axis cs:1,45.2027596515608)
--(axis cs:1,51.3972403484392);

\path [draw=darkorange25512714, semithick]
(axis cs:10,47.201023676115)
--(axis cs:10,53.398976323885);

\path [draw=darkorange25512714, semithick]
(axis cs:100,43.9065511803167)
--(axis cs:100,50.0934488196833);

\path [draw=forestgreen4416044, semithick]
(axis cs:0.1,43.5081411571677)
--(axis cs:0.1,49.6918588428323);

\path [draw=forestgreen4416044, semithick]
(axis cs:1,44.903448111205)
--(axis cs:1,51.096551888795);

\path [draw=forestgreen4416044, semithick]
(axis cs:10,45.0032062025379)
--(axis cs:10,51.1967937974621);

\path [draw=forestgreen4416044, semithick]
(axis cs:100,44.2054895805637)
--(axis cs:100,50.3945104194363);

\path [draw=crimson2143940, semithick]
(axis cs:0.1,38.1493432680814)
--(axis cs:0.1,44.2506567319186);

\path [draw=crimson2143940, semithick]
(axis cs:1,42.2146897958228)
--(axis cs:1,48.3853102041772);

\path [draw=crimson2143940, semithick]
(axis cs:10,42.8114044123583)
--(axis cs:10,48.9885955876417);

\path [draw=crimson2143940, semithick]
(axis cs:100,46.601023676115)
--(axis cs:100,52.798976323885);

\path [draw=mediumpurple148103189, semithick]
(axis cs:0.1,28.6209007658644)
--(axis cs:0.1,34.3790992341356);

\path [draw=mediumpurple148103189, semithick]
(axis cs:1,22.1233546488188)
--(axis cs:1,27.4766453511812);

\path [draw=mediumpurple148103189, semithick]
(axis cs:10,21.3529113350702)
--(axis cs:10,26.6470886649298);

\path [draw=mediumpurple148103189, semithick]
(axis cs:100,30.6724134062337)
--(axis cs:100,36.5275865937663);

\path [draw=sienna1408675, semithick]
(axis cs:0.1,34.4993750650906)
--(axis cs:0.1,40.5006249349094);

\path [draw=sienna1408675, semithick]
(axis cs:1,35.1885060823571)
--(axis cs:1,41.2114939176429);

\path [draw=sienna1408675, semithick]
(axis cs:10,44.6042484053141)
--(axis cs:10,50.7957515946858);

\path [draw=sienna1408675, semithick]
(axis cs:100,46.900967893035)
--(axis cs:100,53.099032106965);

\addplot [semithick, steelblue31119180, dotted, mark=*, mark size=3, mark options={solid}]
table {%
0.1 49.7
1 49.3
10 47.1
100 48.7
};
%\addlegendentry{IIMC}
\addplot [semithick, darkorange25512714, dotted, mark=pentagon*, mark size=3, mark options={solid}]
table {%
0.1 47.6
1 48.3
10 50.3
100 47
};
%\addlegendentry{IIMC}
\addplot [semithick, forestgreen4416044, dotted, mark=square*, mark size=3, mark options={solid}]
table {%
0.1 46.6
1 48
10 48.1
100 47.3
};
%\addlegendentry{IIMC}
\addplot [semithick, sienna1408675, dotted, mark=triangle*, mark size=3, mark options={solid}]
table {%
0.1 41.2
1 45.3
10 45.9
100 49.7
};
%\addlegendentry{IIMC}
\addplot [semithick, crimson2143940, dotted, mark=diamond*, mark size=3, mark options={solid}]
table {%
0.1 31.5
1 24.8
10 24
100 33.6
};
%\addlegendentry{IIMC}
\addplot [semithick, mediumpurple148103189, dotted, mark=oplus, mark size=3, mark options={solid}]
table {%
0.1 37.5
1 38.2
10 47.7
100 50
};
%\addlegendentry{IIMC}
\addplot [semithick, black, dotted, mark=halfcircle, mark size=3, mark options={solid}]
table {%
0.1 33.4
1 33.4
10 33.4
100 33.4
};
\end{axis}

\end{tikzpicture}}}
 
    % \subfloat[\centering Dark Chess]{\scalebox{0.55}{\input{Image/DarkChess_vsPIMC}}}
	% \subfloat[\centering Phantom Tic-Tac-Toe]{\scalebox{0.55}{\input{Image/phantom_ttt_vsPIMC}}}
 	% \subfloat[\centering Dark Hex \hspace{20mm}~]{\scalebox{0.55}{\input{Image/DarkHex_vsPIMC}}}

    \caption{\centering Winning rate of online algorithms. The opponent is PIMC with one second of budget.}

\label{fig:vsPIMCfull}
\end{figure}
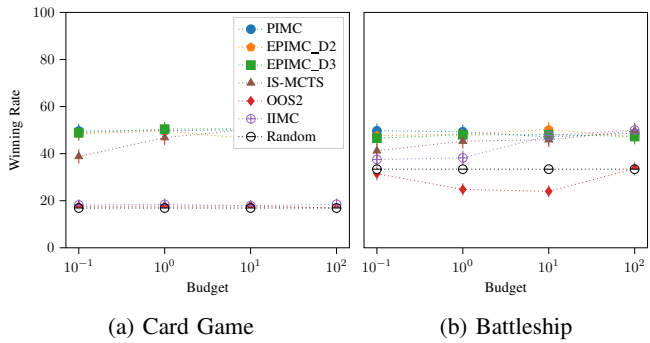

\end{document}